\documentclass[sigconf,nonacm]{acmart}
\AtBeginDocument{%
  }

\usepackage{tabularx}
\usepackage{ragged2e}
\newcolumntype{L}{>{\RaggedRight\arraybackslash}X}
\usepackage{tikz}
\usetikzlibrary{bayesnet} 
\usepackage{enumitem}

\DeclareMathOperator{\tr}{tr}
\DeclareMathOperator{\cov}{Cov}
\DeclareMathOperator{\var}{Var}
\DeclareMathOperator{\diag}{diag}
\newcommand{\norm}[1]{\left\lVert#1\right\rVert}

\begin{document}

%%
%% The "title" command has an optional parameter,
%% allowing the author to define a "short title" to be used in page headers.
\title[The Trade-off Between Covariate Dependence and Latent Structure in Representation Learning]{
The Trade-off Between Covariate Dependence \\and Latent Structure in Representation Learning}

%%
%% The "author" command and its associated commands are used to define
%% the authors and their affiliations.
%% Of note is the shared affiliation of the first two authors, and the
%% "authornote" and "authornotemark" commands
%% used to denote shared contribution to the research.
\author{Małgorzata Łazęcka}
\email{m.lazecka@uw.edu.pl}
\orcid{0000-0003-0975-4274}
\affiliation{%
  \institution{Faculty of Mathematics, Informatics, and Mechanics, University of Warsaw}
  \city{Warsaw}
  \country{Poland}
}

\author{Ewa Szczurek}
\email{em.szczurek@uw.edu.pl}
\orcid{0000-0002-1320-6695}
\affiliation{%
  \institution{Institute of AI for Health, Helmholtz Munich}
  \city{Munich}
  \country{Germany}
}
\affiliation{%
  \institution{Faculty of Mathematics, Informatics, and Mechanics, University of Warsaw}
  \city{Warsaw}
  \country{Poland}
}

%%
%% By default, the full list of authors will be used in the page
%% headers. Often, this list is too long, and will overlap
%% other information printed in the page headers. This command allows
%% the author to define a more concise list
%% of authors' names for this purpose.
\renewcommand{\shortauthors}{Łazęcka M., Szczurek E.}

%%
%% The abstract is a short summary of the work to be presented in the
%% article.
\begin{abstract}
Disentangled representation learning seeks latent representations whose indicidual dimensions each align with a distinct covariate. Unsupervised approaches typically target latent dimension independence, yet this gives no guarantee that the resulting dimensions align with semantically meaningful covariates. Supervised approaches structure the latent space using observed covariates, but under correlated covariates they cannot simultaneously control one-to-one latent-covariate alignment and latent independence. We introduce a unified, supervised framework that couples latent dimension-covariate dependence with constraints on the latent structure. Within this framework, we show 
an inherent trade-off, where enforcing latent independence or exclusive one-to-one latent-covariate dependence comes at a provable cost in latent-covariate alignment. 
We prove that the resulting disentanglement regimes are ordered by the strength of that alignment. Each regime admits a closed-form transformation of the latent space. We apply these transformations post-hoc to realign the representations of pretrained models such as CLIP, DINOv2, and ViT, and we fold them into the inference of informed factor analysis (iFA), a probabilistic model with covariate-informed factors. On simulated and real multi-omics data, we show that both post-hoc alignment and iFA enable controllability of structured latent representations.
\end{abstract}

%%
%% The code below is generated by the tool at http://dl.acm.org/ccs.cfm.
%% Please copy and paste the code instead of the example below.
%%


%%
%% Keywords. The author(s) should pick words that accurately describe
%% the work being presented. Separate the keywords with commas.
%% A "teaser" image appears between the author and affiliation
%% information and the body of the document, and typically spans the
%% page.
% \begin{teaserfigure}
%   \includegraphics[width=\textwidth]{figures/fig0.png}
%   \caption{Seattle Mariners at Spring Training, 2010.}
%   \Description{Enjoying the baseball game from the third-base
%   seats. Ichiro Suzuki preparing to bat.}
%   \label{fig:teaser}
% \end{teaserfigure}

\received{26 July 2027}
\received[revised]{X XXXX 2027}
\received[accepted]{X XXXX 2027}

%%
%% This command processes the author and affiliation and title
%% information and builds the first part of the formatted document.
\maketitle

$ $

\begin{figure}[h!]
    \centering
    \includegraphics[width=\linewidth]{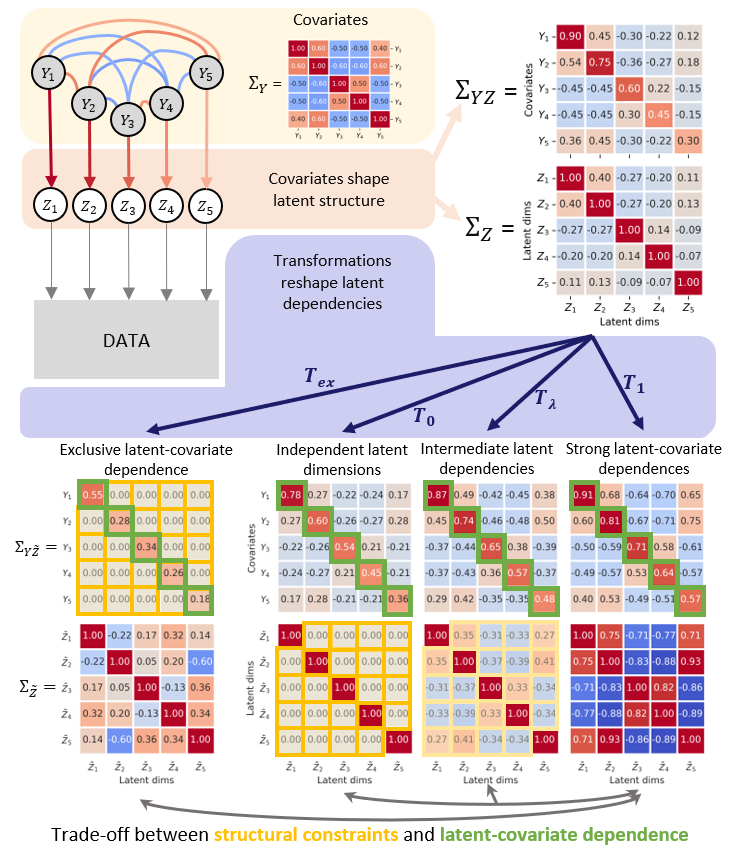}
    \caption{
    Overview of the proposed framework. Correlated covariates ($\Sigma_Y$) shape the latent structure ($\Sigma_Z$). The transformations impose different constraints on the latent space, trading off structural constraints (exclusive latent-covariate dependence, latent inter-independence) against the strength of pairwise latent-covariate dependence ($\Sigma_{Y,\tilde Z}$). 
    }
    \label{fig:fig1}
\end{figure}

\section{Introduction}

Latent representation learning 
maps a high-dimensional data space to a low-dimensional one, encoding each observation as a compact \emph{latent representation} capturing its essential characteristics. 
A long-standing objective of latent representation learning is to obtain representations whose individual dimensions align with distinct, semantically meaningful factors of variation \cite{bengio2013}. Such \emph{disentangled representations} promise interpretability, controllable generation and improved generalization, because they seek to ensure that manipulating a single latent dimension changes only one underlying factor of variation \cite{wang2024}. In many real-world applications, the factors of interest are not latent abstractions, but \emph{observed covariates} such as clinical variables, experimental conditions, or category labels, and the practical aim is a representation in which each latent dimension is aligned one-to-one with its corresponding covariate \cite{higgins2017betavae}.

This assumption holds when the observed covariates are independent, but fails in the practically relevant regime of correlated covariates~\cite{pmlr-v139-trauble21a, dittadi2021on}, as exemplified by $\Sigma_Y$ in Fig.~\ref{fig:fig1}.
This correlation among covariates creates a tension between two objectives:  making
each latent dimension strongly aligned with its covariate, and imposing structural
constraints on the latent space, such as independence among the latent
dimensions or exclusive dependence (each dimension aligning only with its own
covariate). 
This \textbf{trade-off between latent-covariate
dependence and structural constraints on the latent space} has not been
systematically studied.

The principle of latent independence is common in fully unsupervised approaches, and the
way it is applied varies: classical linear methods such as PCA \cite{Pearson01111901, Hotelling1933-gl} and ICA~\cite{HYVARINEN2000411} assume and
recover independent sources of
variation, while modern nonlinear models encourage
independence e.g. through factorized latent representations~\cite{higgins2017betavae, pmlr-v80-kim18b, chen2018isolating}. Enforcing independence alone, however, does not in general
ensure that latent dimensions align with interpretable or meaningful
covariates, nor does it guarantee identifiability~\cite{47692}. Conditioning
latent variable models on observed covariates under additional assumptions restores identifiability up to
simple transformations~\cite{ivae}, and a broad family of nonlinear~\cite{cvae, Zhang2025, ismail2024concept, Hadad_2018_CVPR, 3327345.3327540} and linear~\cite{Bair, BARSHAN, sofa, spear, jafar, Zhou2023, Palzer2022-mx,Samorodnitsky2024}  
supervised representation learning methods uses observed covariates to structure
the latent space. Yet all existing
approaches leave open the problem of the \textbf{simultaneous control of the
one-to-one alignment of each latent dimension with its covariate and of
the structural constraints on the latent space, particularly in the general,
correlated-covariate setting.}

We close this gap with a supervised framework that characterizes the trade-off between latent dimension-covariate dependence and latent structural constraints, providing a principled approach to controlling latent representations. Our contributions are both theoretical and practical. First, (i) leveraging the non-identifiability of latent representations, we formulate optimization problems over invertible \textbf{latent space transformations} ($T$ in Fig.~\ref{fig:fig1}) whose solutions yield latent-covariate dependence across different disentanglement regimes: independent latent dimensions ($T_0$), tunable intermediate latent inter-dependencies ($T_{\lambda}$), unconstrained latent inter-dependencies ($T_{1}$), and a variant with exclusive one-to-one latent-covariate dependence ($T_{\textrm{ex}}$). Second, (ii) we prove that these regimes are \textbf{strictly ordered by latent-covariate alignment strength}, and that there exists a fundamental \textbf{trade-off between structural latent constraints and latent-covariate dependence}.
Finally, (iii) we demonstrate that the proposed transformations can be applied post hoc to \textbf{re-align pretrained representations} or integrated into a \textbf{probabilistic graphical model, informed Factor Analysis (iFA)}. Through experiments on simulated covariate-covariance scenarios and real multi-omics data, we show that both approaches enable controllability of structured  representations.
\section{Related works}

\paragraph{Independence-based disentanglement methods.}  Many disentanglement approaches assume or encourage independence among the dimensions of the representation, an idea rooted in classical latent variable methods such as principal component analysis (PCA) \cite{Pearson01111901, Hotelling1933-gl}, independent component analysis (ICA) \cite{HYVARINEN2000411}, factor analysis (FA) \cite{fa}, and extended to probabilistic latent variable models, including multi-view factor models, e.g. MOFA \cite{mofa} and others \cite{gfa, bass, muvi}. The most prominent deep generative examples are VAE-based methods regularizing the latent distribution toward independence, e.g.  $\beta$-VAE \cite{higgins2017betavae}, FactorVAE \cite{pmlr-v80-kim18b}, or $\beta$-TCVAE \cite{chen2018isolating}. Similar ideas have also been explored beyond VAEs, for example through concept whitening in convolutional neural networks \cite{Chen2020-zm}, InfoGAN \cite{infogan}, a generative GAN model which promotes independence using mutual information, or non-linear ICA extension \cite{nonica}. 

\paragraph{Supervision in disentanglement and latent representation models.} Assuming latent independence alone does not guarantee identifiability \citep{47692}. Consequently, many successful approaches incorporate weak-supervision, such as grouped observations \cite{Bouchacourt2018}, temporal structure \cite{timeica}, or paired observations in which some factors of variation are shared \cite{3524938.3525527}. A second line is supervised in the stronger sense that covariates are directly observed, e.g. used to condition the latent representation, as in CVAE~\cite{cvae},
iVAE~\cite{ivae}, or SE-VAE~\cite{Zhang2025}, among others \cite{ismail2024concept, Hadad_2018_CVPR, 3327345.3327540}. The Identifiable VAE (iVAE)~\cite{ivae} targets identifiability by conditioning a factorized exponential-family prior on the covariates, which identifies the latent representation up to a linear and component-wise transformation, which can be understood as a form of disentanglement.

\paragraph{Supervised linear latent variable models.} The use of covariates to guide representations has been also used in linear models. Early examples include supervised extensions of PCA (SPCA), such as by Bair et al. \cite{Bair}, which focus only on features that are individually associated with the covariates, and Barshan et al. \cite{BARSHAN}, which estimates a sequence of principal components that have maximal dependence with covariates. Other related methods include regression- and correlation-based approaches, such as partial least squares (PLS) \cite{Wold_1975} and canonical correlation analysis (CCA), as well as supervised variants of FA \cite{spear, jafar, Zhou2023, Palzer2022-mx,Samorodnitsky2024}. These methods leverage covariates to improve prediction or interpretability, but generally do not enforce a one-to-one latent-covariate alignment. To the best of our knowledge, the only exception is SOFA \cite{sofa}, which explicitly associates individual latent dimensions with specific covariates. 

\paragraph{Disentanglement under correlated covariates.} 
Disentangling the latent space under correlated covariates is important but
relatively underexplored \citep{pmlr-v139-trauble21a}. Empirical studies that
introduce correlated covariates \citep{dittadi2021on, pmlr-v139-trauble21a} show
that the problem differs fundamentally from the usual assumption of independent
covariates. Enforcing independence between latent dimensions under correlated
covariates degrades the model by e.g. preventing it from attaining the optimal
likelihood \citep{pmlr-v139-trauble21a} or discarding information relevant to the
covariates and thus harming predictive performance \citep{funke2022disentanglement}.
Exact independence is therefore the wrong objective, motivating relaxations. One
line of work requires only the latent space's support to factorize
\citep{roth2023disentanglement}, while another enforces independence conditionally
on the covariates \citep{funke2022disentanglement}. The very definitions and metrics of disentanglement must also be rethought to remain valid under dependence \citep{almudvar2026rethinking}.

\section{Proposed framework}

In this section, we introduce a unified framework that 
%exploits the flexibility of latent representations to describe the 
formally explores the trade-off between latent-covariate dependence and structural constraints on the latent space.

\subsection{Notation and problem setup}
Let $Z \in \mathbb{R}^P$ be a latent random vector with covariance matrix $\Sigma_Z \coloneq \cov(Z)$ and unit marginal variances $
\mathrm{Var}(Z_p) = 1$ for  $p = 1, \ldots, P$. Let $Y \in \mathbb{R}^P$ be a vector of observed covariates, with $\var(Y_p) = 1$ for  $p = 1, \ldots, P$, and define the cross-covariance matrix $\Sigma_{ZY} \coloneq \cov(Z, Y)$. 
%For convenience we assume that both are centered. $Z$ explains data $X$, where $X$. 
Our objective is a latent representation in which every dimension is paired with the covariate, so that we obtain strong alignment in  pairs $(Z_p,Y_p)$ for  $p = 1, \ldots, P$.

We assume that the latent representation $Z$ is non-identifiable, in the sense that for any invertible matrix $T \in \mathbb{R}^{P \times P}$, the linearly transformed vector $\tilde Z = TZ$ provides an equivalent representation of the data $X$.
We exploit this freedom to select a transformation $T$ that imposes a desirable structure on the latent space, %and thereby resolves representation ambiguity by selecting a canonical structured representation. 
thereby defining a canonical structured representation and resolving representation ambiguity. 
Here $X$ may be a single
vector $X \in \mathbb{R}^d$ or, in the multi-modal case, a collection of vectors
$\{X^{(m)} \in \mathbb{R}^{d^{(m)}}\}_{m=1}^{M}$ across $M$ modalities. 

% In particular, we consider three regimes: \textbf{(i) decorrelated/independent latent dimensions}, which induce independence under a joint Gaussian assumption (and uncorrelatedness without assuming gaussianity) among the latent dimensions while aligning each latent dimension with the corresponding covariate; \textbf{(ii) strong dependence with covariates}, which maximizes the correlation between each latent dimension and its corresponding covariate without imposing structural constraints on the latent space; and \textbf{(iii) moderate dependencie}, which interpolates between these two extremes. We
% additionally examine a fourth scenario, \textbf{(iv) exclusive latent-dimension-covariate
% dependence}.
In particular, we consider three disentanglement regimes: \textbf{(i) strong latent dependence
on covariates}, which maximizes the correlation between each latent
dimension and its covariate while leaving the latent
dependence structure otherwise unconstrained, \textbf{(ii)
decorrelated/independent latent dimensions}, which enforces
uncorrelatedness among the latent dimensions (independence under a joint
Gaussian assumption) while aligning each latent dimension with its covariate, and \textbf{(iii) intermediate latent inter-dependence}, which
weights the alignment objective against proximity to the decorrelated
solution and thereby interpolates between the two extremes. We
additionally examine a fourth scenario, \textbf{(iv) exclusive
latent-dimension-covariate dependence}, in which each latent dimension
is required to correlate only with its own covariate. 

%We consider the problem of framing (i)--(iv)  within a single unified setup, and and postulate a provable ordering of all four regimes by their covariate dependence.

%Individually, regimes (i) and (ii) reduce to classical linear constructions, as we show below. Our starting point, however, is different: 
We pose the problem as pairwise latent-covariate dependence
under structural constraints on the latent space,
%and the resulting extremal transformations coincide with familiar linear methods. The contribution lies in 
by framing (i)-(iv) within a single unified setup.
In particular, in following section we prove that regime (iii) arises as a continuous interpolation between
(i) and (ii), and an ordering ordering of all four regimes by their alignment strength.
%This framework reveals a fundamental trade-off between two common objectives in representation learning: maximizing dependence with observed covariates and imposing structural constraints on the latent representation, such as independence among latent dimensions.
%This characterizes how structural constraints commonly assumed in
%disentanglement limit another of its goals: the achievable
%strength of covariate-latent dependence.

\subsection{Optimization-based formulations and solutions}
\label{sec:T_optimization}
%In the following, we give an optimization-based formulation of
%regimes~(i)-(iv), derive their solutions in closed form, and rank them
%by the strength of the latent-covariate dependence they achieve. 
We begin with the independence regime~(ii), since its
whitening transformation fixes a canonical representation
$\overline{Z}$ on which all remaining solutions are expressed, then we
solve the unconstrained regime~(i), followed by the interpolation
(iii), and finally the exclusive-dependence regime~(iv).

\subsubsection{Independent latent dimensions}

We first consider the problem of finding a whitening transformation $T$ such that the transformed latent variables are uncorrelated and standardized
$\cov(TZ) = I$, or equivalently $T \Sigma_Z T' = I$.
This constraint alone does not determine $T$ uniquely: if $T$ satisfies the above condition, then so does $QT$ for any orthogonal matrix $Q$ \cite{Kessy2018}, so whitening is defined only up to an orthogonal
rotation. A canonical choice for $T$ is the symmetric whitening transformation $\Sigma_Z^{-1/2}$, the unique symmetric positive-definite matrix that whitens $Z$. The general solution has a form
\[
T = Q \Sigma_Z^{-1/2}, \quad Q' Q = I.
\]
To resolve the remaining ambiguity, we incorporate the covariates $Y$ and seek a transformation that, in addition to whitening, maximizes the alignment in latent dimension-covariate pairs $(\tilde{Z}_p, Y_p)$, which leads to 
\begin{equation}
\label{eq:optim_indep}
    \max_T \ \tr(T \Sigma_{ZY})
\quad \text{s.t.} \quad T \Sigma_Z T' = I.
\end{equation}
Substituting $T$ with $Q \Sigma_Z^{-1/2}$ reduces the problem to
\[
\max_{Q' Q = I} \ \tr\big(Q \Sigma_Z^{-1/2} \Sigma_{ZY}\big),
\]
which is the orthogonal Procrustes problem \cite{Schnemann1966} (Section \ref{sec:procrustes}). Let
\[
\Sigma_Z^{-1/2} \Sigma_{ZY} = U D V'
\]
be the singular value decomposition. Then the optimal solution is
\[
Q^* = VU', \quad T_{0}^* = Q^* \Sigma_Z^{-1/2}.
\]
In terms of $\Sigma_Z$ and $\Sigma_{ZY}$ (Section \ref{sec:sigma_formula}) $T_{0}^*$ equals
\begin{equation}
\label{eq:t0}
    T_0^* = \big(\Sigma_{ZY}'\Sigma_Z^{-1}\Sigma_{ZY}\big)^{-1/2}\,\Sigma_{ZY}'\Sigma_Z^{-1}.
\end{equation}
Note that \eqref{eq:t0} is well-defined whenever $\Sigma_Z$ is positive
definite.

In the following sections, we work with latent variables that have been whitened by the transformation $T_0^*$, and we denote the transformed $Z$ by $\overline{Z} \coloneq T_0^* Z$. This simplifies the algebra while preserving the generality of the results, as to recover the transform on the original $Z$ we simply compose the optimal transformations $T(\overline{Z})$ with $T_0^*$ (justification that this yields the same solution as solving the problem for a general $Z$ is given in Section~\ref{sec:identifiability}).

% For notational convenience, we introduce a whitened latent $\bar{Z} \coloneq T_0^* Z$, as it simplifies the derivation of the subsequent regimes. To recover the transformation on the original $Z$, we simply compose the optimal transformation $T(\bar{Z})$ with $T_0^*$: by non-identifiability of the latent (Section~\ref{sec:identifiability}), this composition solves the corresponding problem on $Z$ directly.

\subsubsection{Strong latent dependence on covariates}
We now consider a formulation that prioritizes alignment with $Y$ while relaxing the demand on the latent independence. In particular, we only enforce that the transformed latent variables have unit marginal variances, thus the problem is
\begin{equation}
\label{eq:optim_max}
    \max_T \ \tr(T \Sigma_{\overline{Z}Y})
\quad \text{s.t.} \quad \diag(T T') = (1, \ldots, 1)'.
\end{equation}
Let $T = [t_1, \ldots, t_P]'$, where $t_p \in \mathbb{R}^P$ are the rows of $T$. The constraint reduces to $\norm{t_p} = 1$ for all $p$, and the objective term decomposes as
\[
\tr(T \Sigma_{\overline{Z}Y}) = \sum_{p=1}^P t_p' (\Sigma_{\overline{Z}Y})_{\cdot p}.
\]
Denoting $a_p \coloneq (\Sigma_{\overline{Z}Y})_{\cdot p}$, the problem separates into $P$ independent subproblems:
\[
\max_{\|t_p\| = 1} \ t_p' a_p.
\]
By the Cauchy-Schwarz inequality,
\[
t_p' a_p \leq \|t_p\| \, \|a_p\|,
\]
with equality if and only if $t_p$ is collinear with $a_p$. Hence, the optimal solution is $
t_p^* =
\frac{a_p}{\norm{a_p}}$ for $\norm{a_p} \neq 0$, or equivalently,
\begin{equation}
\label{eq:t1}
    t_p^* = \frac{(\Sigma_{\overline{Z}Y})_{\cdot p}}{\|(\Sigma_{\overline{Z}Y})_{\cdot p}\|}.
\end{equation}
For general $Z$, the corresponding transformation is
$T_1^* =  T_1^*(\overline{Z})\, T_0^*.$ The connection to regression is discussed in Section~\ref{sec:notes_t1}.

\subsubsection{Intermediate scenarios}
We now interpolate between the two regimes. Let $\lambda \in [0,1]$ and consider
\begin{equation}
\label{eq:optim_lambda}
    \max_T \lambda \tr(T \Sigma_{\overline{Z}Y}) - \frac{1-\lambda}{2}\|T - I\|_F^2
\textrm{ s.t. } \diag(T T') = (1, \ldots, 1)',
\end{equation}
For $\lambda = 0$, $T^* = I$, while for $\lambda = 1$ we recover \eqref{eq:optim_max}.
Using the identity
\[
\|T - I\|_F^2 = \sum_{p=1}^P \|t_p - e_p\|^2,
\]
where $e_p$ is the $p$-th canonical basis vector, the problem decomposes into:
\[
\max_{\|t_p\| = 1} \lambda \ t_p' a_p - \frac{1-\lambda}{2}\|t_p - e_p\|^2,
\]
where $a_p := (\Sigma_{\overline{Z}Y})_{\cdot p}$.
Expanding the quadratic term and using $\|t_p\| = 1$, the objective is equivalent to (up to a constant independent of $t_p$)
\[
\max_{\|t_p\| = 1} \ \lambda\, t_p' a_p + (1-\lambda)\, t_p' e_p.
\]
By the Cauchy-Schwarz inequality, the optimum is attained when $t_p$ is collinear with $\lambda a_p + (1-\lambda)e_p$, which results in
\begin{equation}
\label{eq:tlambda}
t_{p}^* = \frac{\lambda (\Sigma_{\overline{Z}Y})_{\cdot p} + (1-\lambda) e_p}
{\|\lambda (\Sigma_{\overline{Z}Y})_{\cdot p} + (1-\lambda) e_p\|}.
\end{equation}
For general $Z$, we obtain $T_\lambda^* = T_\lambda^*(\overline{Z}) \, T_0^*.$
$T_\lambda^*$ recovers the previous solutions $T_0^*$ and $T_1^*$  for $\lambda=0$ and $\lambda=1$, respectively (see Table \ref{tab:summaryt}).

\subsubsection{Exclusive latent dependence with covariates}

We now require each latent dimension to be associated with its covariate while being
cross-uncorrelated with all others, i.e. the cross-covariance $\cov(T\overline{Z}, Y) =
T\Sigma_{\overline{Z}Y}$ is diagonal
\begin{equation}
\label{eq:optim_excl}
    \max_T \ \tr(T\Sigma_{\overline{Z}Y})
    \textrm{ }\text{s.t.}\textrm{ } (T\Sigma_{\overline{Z}Y})_{pq}=0\ (p\neq q),
    \textrm{ }\diag(T T')=(1,\dots,1)'.
\end{equation}
With $T=[t_1,\dots,t_P]'$ and $a_p\coloneq(\Sigma_{\overline{Z}Y})_{\cdot p}$, each row solves
\[
\max_{\|t_p\|=1}\ t_p'a_p
\quad\text{s.t.}\quad t_p'a_q=0\ \ (q\neq p).
\]
The constraints force $t_p$ to be orthogonal to all $a_q$ for $q\neq p$, so the optimum lies in a one-dimensional subspace spanned by  $(\Sigma_{\overline{Z}Y}^{-1})_{p \cdot}$ as $(\Sigma_{\overline{Z}Y}^{-1})_{p \cdot}(\Sigma_{\overline{Z}Y})_{\cdot q} = \mathbb{I}(p=q)$. Thus, taking into account the second constraint, we get
\begin{equation}
\label{eq:tex}
t_p^* = \frac{(\Sigma_{\overline{Z}Y}^{-1})_{p \cdot}'}{\norm{(\Sigma_{\overline{Z}Y}^{-1})_{p \cdot}}}
\end{equation}
giving
$T_{\mathrm{ex}}^* =T_{\mathrm{ex}}^*(\overline{Z}) T_{0}^*$
for a general $Z$. Note that here invertibility of $\Sigma_{ZY}$ is required for feasibility.

\subsubsection{The comparison of the strength of the dependencies of the regimes}

Theorem~\ref{lemma1} orders the four regimes by latent-covariate dependence strength (proof in Section~\ref{sec:proof}): dependence increases from the independent regime~(ii) through the intermediate regime~(iii) to the unconstrained regime~(i), while the exclusive regime~(iv) attains weaker alignment than all of~(i)-(iii). The result makes the trade-off between latent-covariate dependence and structural constraints on the latent precise: in particular latent independence constraint in (ii) carries a cost in alignment relative to the unconstrained regime (i), and the intermediate regime~(iii) traces the entire spectrum of solutions between the two extremes. For general $Z$, the regime regularizes the transformation using $\|T(T_0^*)^{-1}-I\|_F$, the deviation from the decorrelating solution. This choice makes the problem separable and admits the closed-form solutions $T_\lambda$, which is not the case for the latent-correlation penalty $\|\cov(\tilde{Z})-I\|_F$. 
%To our knowledge, no prior work connects the latent independence-covariate dependence tension in latent representations to a provable ordering over a continuum of regimes. 
All four transformations and their per-dimension covariate dependences are summarised in Section~\ref{sec:summary_of_t}.

\begin{theorem}
\label{lemma1}
(a) Assume $\Sigma_Z$ is positive definite and that $\Sigma_{ZY}$ has full rank.  Let
$
J(T)\coloneq \tr\!\big(\cov(\tilde Z, Y)\big)$
denote the total latent dimension-covariate alignment, where $\tilde Z = TZ$. Then the optimal
transformations of the four regimes are ordered by their alignment strength,
\[
J(T_{\mathrm{ex}}^*) \;\le\; J(T_0^*) \;\le\; J(T_\lambda^*) \;\le\; J(T_1^*),
\qquad \lambda \in [0,1]
\]
and $\lambda \mapsto J(T_\lambda^*)$ is non-decreasing, so the intermediate regime interpolates
$J(T_0^*)$ and $J(T_1^*)$. Moreover, all three inequalities hold with equality simultaneously if and only if
$\big(\Sigma_{ZY}'\Sigma_Z^{-1}\Sigma_{ZY}\big)^{1/2}$ is diagonal, i.e. the whitened cross-covariance columns are mutually orthogonal. \\
(b) Let $F(T) = \|\cov(\tilde{Z} - I\|_F =\|T \Sigma_Z T^\top - I\|_F$
measure the deviation from mutually decorrelated unit-variance latent dimensions, then
\begin{equation*}
    F(T_0^*) \;\leq\; F(T_1^*) \quad \textrm{and} \quad F(T_0^*) \;\leq\; F(T_{\mathrm{ex}}^*).
\end{equation*}
\end{theorem}

\subsection{Statistical model}

We now embed the framework in a probabilistic factor analysis model that uses the covariates to learn informed factors and the transformation family $T_\lambda$ to reshape them.

\subsubsection{Factor analysis with pairwise dependencies with covariates}

\label{sec:population_model}
We consider a random vector $X \in \mathbb{R}^D$ of $D$ features following a factor model
\[
X | Z,W,\tau \sim \mathcal{N}(W Z, D_{1/\tau}),
\]
 where $Z$ is a latent representation, with each component corresponding to a  factor, and $W \in \mathbb{R}^{D \times P}$ is a matrix of loadings. The matrix $D_{1/\tau} = \mathrm{diag}(1/\tau_1,\ldots,1/\tau_D)$ encodes feature-specific noise variances, where $\tau_d$ denotes the precision of the $d$-th observed feature.
 In this model, with no further constraints on $W$ or $Z$, the representation is invariant to invertible linear transformations, as
\[
WZ = (W T^{-1})(T Z) = \tilde{W}\tilde{Z},
\]
for any invertible matrix $T \in \mathbb{R}^{P \times P}$.

We extend this framework by introducing a vector of covariates $Y \in \mathbb{R}^P$, with covariance matrix $\Sigma_Y := \mathrm{Cov}(Y)$, assumed to have unit diagonal entries.
Each latent factor is then modeled as being \emph{informed} by the corresponding covariate through a pairwise linear relationship (see Fig. \ref{fig:PGM})
\[
Z|Y\sim \mathcal{N}(\beta^{(0)} + D_{\beta} Y, D_{1-\beta^2}),
\]
where $\beta^{(0)} \in \mathbb{R}^P$ is a vector of intercepts, $D_\beta = \mathrm{diag}(\beta_1,\ldots,\beta_P)$, and coefficients $\beta_p \in [0,1]$. The matrix $D_{1-\beta^2} = \mathrm{diag}(1-\beta_1^2,\ldots,1-\beta_P^2)$ ensures unit variance of each latent factor. The full covariance of $Z$ is given by
\[
\Sigma_Z \coloneq \mathrm{Cov}(Z) = D_\beta \Sigma_Y D_\beta + D_{1-\beta^2}
\]
which shows that latent factors are generally dependent due to the dependence structure of $Y$. However, conditionally on $Y$, the factors become independent since their conditional covariance matrix is diagonal:
\[
\Sigma_{Z|Y} \coloneq\mathrm{Cov}(Z|Y)= D_{1-\beta^2}.
\]
Under this model, the cross-covariance between latent factors and covariates is given by
\[
\Sigma_{ZY} \coloneq \mathrm{Cov}(Z,Y) = D_\beta \Sigma_Y.
\]

\subsubsection{Informed factor analysis (iFA)}

We introduce a probabilistic
graphical model that follows the generative formulation of
Section~\ref{sec:population_model} augmented with additional uninformed latent
factors (Fig.~\ref{fig:PGM}) and using transformations (i)-(iii) introduced in Section \ref{sec:T_optimization}. Let $\mathbf{X}\in\mathbb{R}^{N\times D}$ denote
the observed features, $\mathbf{Y}\in\mathbb{R}^{N\times P}$ the covariates
(assumed standardized to unit variance per column), and
$\mathbf{Z}\in\mathbb{R}^{N\times K}$ the latent factors with $P\leq K$. The
first $P$ factors are informed by the covariates as in
Section~\ref{sec:population_model}; the remaining $K-P$ are uninformed and
follow a standard Gaussian prior. The invertible transformation $T\in\mathbb{R}^{K\times K}$ is applied to factors $\tilde{z}_{n,\cdot} = z_{n,\cdot}T'$, and we take $T$ to act
trivially on the uninformed block, so $\tilde z_{n,P+1:K} = z_{n,P+1:K}$. We place an automatic relevance determination
(ARD) prior on the loadings $W\in\mathbb{R}^{D\times K}$ to encourage factor-level
sparsity, and Gamma priors on the feature-specific noise precisions $\tau_d$.
The joint distribution factorizes as
\begin{align}
\label{eq:joint_likelihood}
    &p(\mathbf{X}, \tilde{\mathbf{Z}}, W, \alpha, \tau \mid\mathbf{Y}, \beta) = \prod_{n=1}^{N}\prod_{d=1}^{D} \mathcal{N}\bigl(x_{n,d}\mid \tilde{z}_{n,\cdot} w_{d,\cdot}',\, 1/\tau_d\bigr)\\
    &\quad \prod_{n=1}^{N} \prod_{p=1}^{P} \mathcal{N}(z_{n,p}\mid\beta_p^{(0)} + \beta_p y_{n,p}, 1-\beta_p^2)\prod_{n=1}^{N} \prod_{k=P+1}^{K} \mathcal{N}(z_{n,k}\mid 0,1) \nonumber \\
    &\quad \prod_{d=1}^{D} \prod_{k=1}^{K} \mathcal{N}(w_{d,k} \mid 0, 1/\alpha_k)\prod_{k=1}^{K}\mathcal{G}(\alpha_k \mid a_0^{(\alpha)}, b_0^{(\alpha)})  \nonumber\\
    &\quad  \prod_{d=1}^{D} \mathcal{G}(\tau_d \mid a_0^{(\tau)}, b_0^{(\tau)}). \nonumber
\end{align}
In the supplement Eq.~\eqref{eq:multimodal} we provide the joint distribution for multi-modal data.

\subsubsection{Inference}
The posterior $p(\tilde{\mathbf{Z}}, W, \alpha, \tau\mid \mathbf{X}, \mathbf{Y}, \beta, \beta^{(0)})$ is approximated by variational inference \cite{vi_book, vi_intro} with the mean-field factorization
\begin{align*}
    q(\tilde Z, W, \alpha, \tau \mid Y, \beta) =& \prod_{n=1}^{N}q(\tilde z_{n, 1:P}) \prod_{n=1}^{N}\prod_{k=P+1}^{K}q(\tilde z_{n, k}) \\
    & \prod_{d=1}^{D}\prod_{k=1}^{K}q(w_{d, k})\prod_{k=1}^{K}q(\alpha_{k})\prod_{d=1}^{D}q(\tau_{d}).
\end{align*}
Variational inference chooses $q$ to
maximize the evidence lower bound (ELBO)
\begin{equation*}
    \mathcal{L}(q,\beta,\beta^{(0)})
    = \mathbb{E}_q\!\bigl[\log p(\mathbf{X}, \mathbf{Z}, W, \alpha, \tau\mid
        \mathbf{Y}, \beta, \beta^{(0)})\bigr] - \mathbb{E}_q[\log q],
\end{equation*}
which satisfies
$\log p(\mathbf{X}\mid \mathbf{Y}, \beta, \beta^{(0)})
   = \mathcal{L} + \mathrm{KL}(q\,\|\,p) \ge \mathcal{L}$.
We maximize $\mathcal{L}$ by coordinate ascent: for each latent variable, the optimal variational parameters are obtained from
\[\log q^*(\theta_j) = \mathbb{E}_{q_{-j}}\log p + \text{const},\]
while $\beta$ and $\beta^{(0)}$ are updated by maximizing $\mathcal{L}$ with
respect to them as point parameters. The updates for
$q(\alpha_k)$ and $q(\tau_d)$ are standard in Bayesian factor analysis
(e.g. \cite{mofa}). Below we provide the updates specific to the covariate-informed
setting. 

The two natural choices of the parameterization of latent factors lead to different conveniences.
The updates for $W$ and $\tau$ are simpler when written in terms of $\tilde z$,
since the likelihood depends on the factors only through $\tilde z$. The
$\beta_p^{(0)}, \beta_p$ updates, in contrast, are simpler in terms of $z$,
since the prior is defined on $z$. Throughout, we use $\tilde z$ as the
variational variable and recover the moments of $z$ via the inverse
transformation $T^{-1}$.
\paragraph{Updates for informed factors $\tilde{z}_{n,1:P}$.} The conditional log-density is a quadratic form in $\tilde{z}_{n,1:P}$, so $q^*(\tilde{z}_{n,1:P}) = \mathcal{N}(\tilde{\mu}_n, \tilde{\Sigma}_n)$ with
\begin{align*}
    \tilde{\Sigma}_n^{-1} &= (T')^{-1}D_{1/(1-\beta^2)}T^{-1} + \sum_{d=1}^{D} \langle \tau_d\rangle\, \langle w_{d,1:P}\, w_{d,1:P}'\rangle, \\
    \tilde{\mu}_n &= \tilde{\Sigma}_n\Bigl[\, T^{-1}D_{1/(1-\beta^2)}\bigl(\beta^{(0)} + D_\beta y_n\bigr) \\
    &\quad + \sum_{d=1}^{D}\langle \tau_d\rangle\,  \bigl(x_{n,d} - \langle w_{d,P+1:K}\rangle'\langle \tilde z_{n,P+1:K}\rangle\bigr)\langle w_{d,1:P}\rangle\Bigr],
\end{align*}
where $\langle\cdot\rangle$ denotes expectation under $q$ and %$D_{1/(1-\beta^2)} = \mathrm{diag}\bigl(1/(1-\beta_p^2)\bigr)$.
$D_{1/(1-\beta^2)} = \mathrm{diag}\bigl(\tfrac{1}{1-\beta_p^2}\bigr)$.
\paragraph{Updates for $\beta_p^{0}$ and $\beta_p$}
The ELBO contribution that depends on $\beta_p^{(0)}$ and $\beta_p$ is
\begin{equation*}
    \mathcal{L}_p(\beta_p^{(0)},\beta_p) = -\tfrac{N}{2}\log(1-\beta_p^2) - \frac{\sum_{n=1}^N\mathbb{E} (z_{n,p} - \beta_p^{(0)}-\beta_p y_{n,p})^2}{2(1-\beta_p^2)},
\end{equation*}
Maximizing with respect to $\beta_p^{(0)}$ gives $\beta_p^{(0)} \;=\; \overline{\langle z_p\rangle} - \beta_p\, \bar y_p$,
with $\overline{\langle z_p\rangle} = N^{-1}\sum_n \langle z_{np}\rangle$ and
$\bar y_p = N^{-1}\sum_n y_{np}$. Substituting this expression back into
$\mathcal{L}_p$ yields
\begin{equation*}
    \mathcal{L}_p(\beta_p) = -\tfrac{N}{2}\log(1-\beta_p^2) - \frac{S_z^{(p)} - 2\beta_p\, C_{zy}^{(p)} + \beta_p^2\, N}{2(1-\beta_p^2)},
\end{equation*}
with 
\begin{align*}
    S_z^{(p)} &= \sum_n \langle (z_{np}-\overline{\langle z_p\rangle})^2\rangle, \\
    C_{zy}^{(p)} &= \sum_n (\langle z_{np}\rangle - \overline{\langle z_p\rangle})(y_{np}-\bar y_p).
\end{align*}
The stationarity equation is cubic in $\beta_p$
\begin{equation*}
    N\beta_p^3 - (1+\beta_p^2)\,C_{zy}^{(p)} + \beta_p\, S_z^{(p)} = 0,
\end{equation*}
which we solve numerically (one-dimensional root-finding within $(0,1)$). When the variational moments are close to the prior $S_z^{(p)} \approx N$, we get the closed-form approximation $\beta_p \approx C_{zy}^{(p)}/N$, i.e., the empirical correlation between $\langle z_p\rangle$ and $y_p$.
\paragraph{Updates for loadings $w_{dk}$.}
The conditional log-joint is quadratic in each $w_{dk}$, giving
$q^*(w_{dk}) = \mathcal{N}\!\bigl(\mu_{dk}^{(w)},\, (\sigma_{dk}^{(w)})^2\bigr)$
with precision and mean
\begin{align*}
    (\sigma_{dk}^{(w)})^{-2}
    &= \langle\alpha_k\rangle + \langle\tau_d\rangle \sum_{n=1}^{N}\langle \tilde{z}_{nk}^2\rangle, \\
    \mu_{dk}^{(w)}
    &= (\sigma_{dk}^{(w)})^{2}\,\langle\tau_d\rangle
       \left[\sum_{n=1}^{N}\langle \tilde{z}_{nk}\rangle\, x_{nd}
       - \sum_{j\ne k}\langle w_{dj}\rangle \sum_{n=1}^{N}\langle \tilde{z}_{nj} \tilde{z}_{nk}\rangle\right].
\end{align*}
The cross-moment $\langle \tilde{z}_{nj} \tilde{z}_{nk}\rangle$ equals
$\langle \tilde{z}_{nj}\rangle\langle \tilde{z}_{nk}\rangle$ whenever $j$ or $k$ lies outside
the informed block (mean-field independence), and equals
$\langle \tilde{z}_{nj}\rangle\langle \tilde{z}_{nk}\rangle + (\tilde{\Sigma}_n)_{jk}$ when both
$j, k\le P$, with $\tilde{\Sigma}_n$ as defined above. Substituting these expressions
makes the dependence on the informed-block covariance explicit, and reduces to
the standard Bayesian factor analysis update when $P = 0$.

\subsubsection{Transformation application}
The transformations we want to apply ($T_\lambda^*$ of Section~\ref{sec:T_FA}) depend on
$(\beta, \Sigma_Y)$, so we do not optimise $T$ and $\beta$ jointly. Instead,
we use a two-stage procedure.

\textbf{Stage 1.} We fit the model with $T = I$ by coordinate-ascent VI as
described above. This yields $\hat q^{(1)}(z_{n,1:P}) = \mathcal{N}(\mu_n,
\Sigma_n)$, point estimates $\hat\beta, \hat\beta^{(0)}$, and the standard
posteriors for $W, \alpha, \tau$.

\textbf{Computing $T_\lambda^*$.} From $\hat\beta$ and the empirical covariate
covariance $\hat\Sigma_Y$ we form
$\hat\Sigma_Z = D_{\hat\beta}\,\hat\Sigma_Y\,D_{\hat\beta} + D_{1-\hat\beta^2}$ and $\hat\Sigma_{ZY} = D_{\hat\beta}\,\hat\Sigma_Y$,
and plug them into the closed-form expressions of
Section~\ref{sec:T_optimization}. We block-extend $T_\lambda^*$ to act as the
identity on the uninformed factors, so $T \in \mathbb{R}^{K\times K}$ with
$T_{1:P, 1:P} = T_\lambda^*$ and $T_{P+1:K, P+1:K} = I$.

\textbf{Stage 2.} We replace the variational means of the informed block
with their $T$-transformed versions,
$\tilde\mu_n \leftarrow T_{1:P}\,\mu_n$. We then hold $\tilde\mu_{n,1:P},
\hat\beta, \hat\beta^{(0)}, T$ fixed and continue coordinate-ascent VI for the rest of variational parameters until convergence.

Stage~2 is a relatively cheap fine-tuning step (standard FA with some factor moments fixed) that can be repeated, for the
same Stage~1 fit, across many values of $\lambda$. It lets the loadings $W$ adapt to the new basis. 

Although the generative model assumes continuous $Y$, the framework extends to binary $y_p$ by replacing correlation coefficient with point-biserial correlations.

%\section{Numerical experiments}

\section{Experimental setup}

Our experiments have two goals: to empirically show the theoretical properties of the transformed latent spaces on simulations, and to demonstrate the framework's practical application as a post-hoc transformation and inside iFA on simulations and real-world data.

\subsubsection{Artificial data} We designed four data-generation scenarios, following the generative process of the model defined in Eq.~\eqref{eq:joint_likelihood} (see Fig.~\ref{fig:PGM} for graphical representation) with $T=I$, each targeting a different realistic covariance structure (Tab.~\ref{tab:sup_scenarios}) among the normally distributed covariates: 1.~mixed positive and negative correlations (PN) (e.g.\ blood-panel markers, where some pairs move together and others in opposite directions), 2.~autoregressive structure (AR) (e.g.\ measurements taken consecutively over time), 3.~purely positive correlations (P) (e.g.\ bone measurements that all scale with overall body size), and 4.~purely negative correlations (N) (e.g.\ dummy-coded categories to which patients are assigned). Across all scenarios we held fixed the number of features ($D=100$), the latent dimension ($K=10$), the number of covariates ($P=5$), the factor-covariate dependence strengths $\beta=(0.9, 0.75, 0.6, 0.45, 0.3)$, and the noise-to-signal ratio ($\theta=0.1$) (Tab.~\ref{tab:sup_sim_par}). We vary two parameters: $\alpha$, controlling covariate correlation strength (from 0 to strong dependence), and $b$, scaling the $\beta$ vector (Tab. \ref{tab:sup_vary_par}). In sample-based experiments we used $N=500$ observations and $20$ repetitions.

\subsubsection{Large model's representations} As a testbed for post-hoc alignment on realistic latent space, we use
image representations from three large pretrained models spanning the main
pretraining paradigms: \texttt{CLIP}~\cite{Radford2021LearningTV} (contrastive
image-text), \texttt{DINOv2}~\cite{oquab2024dinov} (self-supervised
self-distillation), and \texttt{ViT}~\cite{dosovitskiy2021an}
(supervised classification on ImageNet). As input data, we use
 a validation subset of the Tiny-ImageNet dataset \cite{Le2015TinyIV}, which is a downsampled version of ImageNet \cite{5206848}, covering $200$ classes
with $50$ validation images per class. From this subset, we retain $2000$ observations and use the $P=20$ dummy-coded class indicators as covariates corresponding to animal categories (see Fig.~\ref{fig:sup_emb} for the exact labels).
%, resulting in a dataset where class labels are available for half of the observations. 
For each
observation we extract the representations from each of the three
pretrained models.

\begin{figure}
    \centering
    \includegraphics[width=1\linewidth]{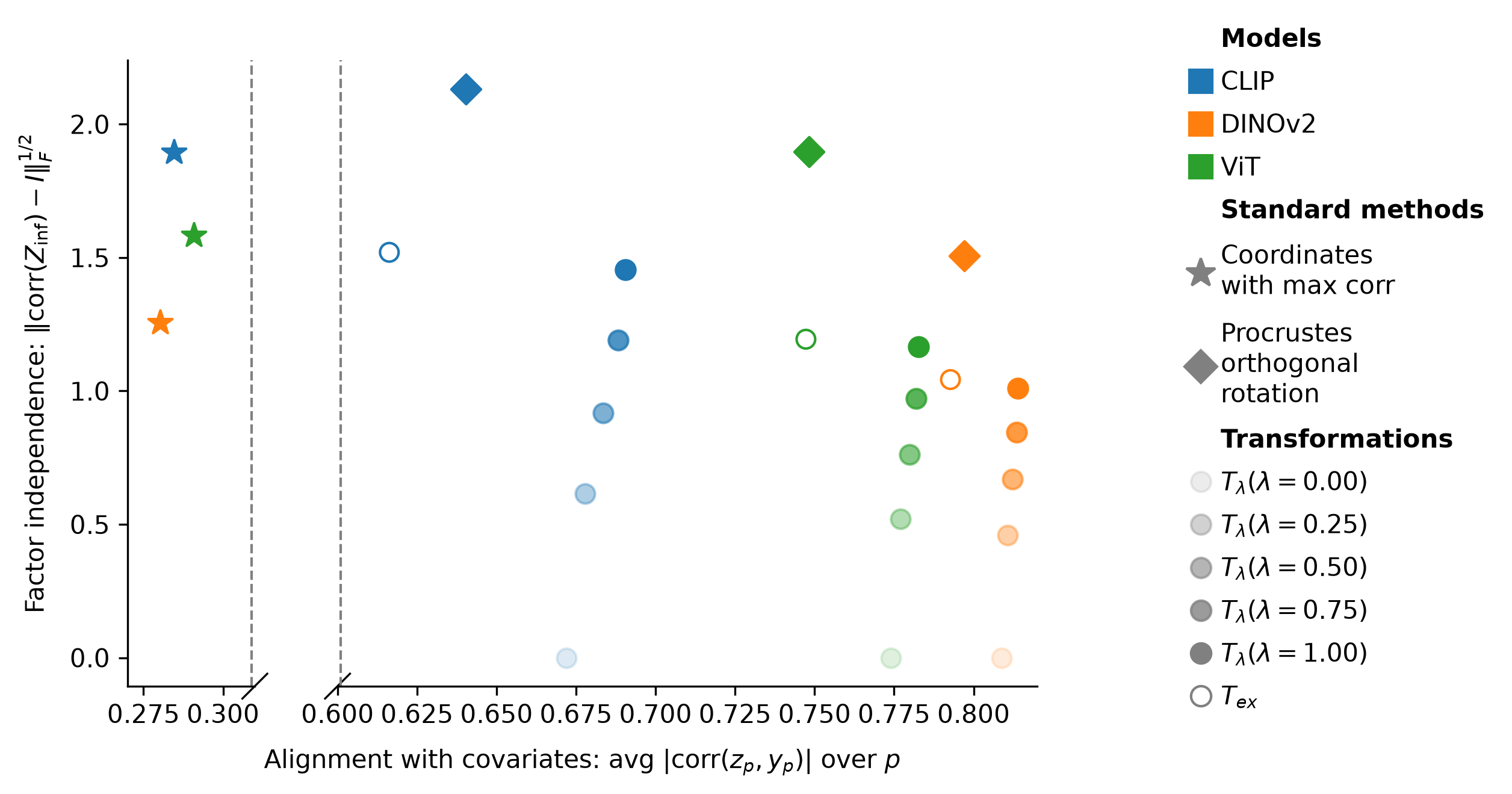}
    \caption{
    Post-hoc latent space transformations of pretrained representations (CLIP, DINOv2, and ViT): \textit{dimensions with max corr} (latent dimensions most strongly correlated with the covariates) and \textit{Procrustes orthogonal rotation} against the proposed transformation applied after Procrustes rotation.}    
    \label{fig:big_models}
\end{figure}

\subsubsection{Real-world data} We used the breast TCGA dataset \cite{Cancer_Genome}, preprocessed by \cite{10.1371/journal.pcbi.1005752}, consisting of multi-modal data (mRNA, miRNA, and proteomics) from $150$ breast cancer patients categorized into cancer subtypes (Basal, Her2, and LumA), which we used as dummy-coded covariates. We used $5$-fold cross-validation, ensuring a stratified distribution of cancer subtypes across folds, and standardized the data matrices within each fold prior to analysis.

\subsubsection{Comparison methods} We compare against methods divided into three categories. \emph{Unsupervised:} \texttt{PCA}, which assumes full latent independence, and \texttt{MOFA} \cite{mofa}, which uses a mean-field approximation that factorizes the variational distribution over factors, approximating rather than strictly enforcing their independence. \emph{Supervised but not paired:} two supervised extensions of \texttt{PCA}, \texttt{SPCA (Bair)} \citep{Bair} and \texttt{SPCA (Barshan)} \citep{BARSHAN}, \texttt{PLS}, which extracts components maximizing the data-covariate covariance, and \texttt{iVAE} \cite{ivae}, a VAE model with covariate-conditioned prior. \emph{Supervised and paired:} \texttt{Y as factors} takes the covariates directly as factors, $Z_p = Y_p$, achieving perfect alignment by construction but not controlling for independence, and \texttt{SOFA} \cite{sofa}, an extension of \texttt{MOFA} that models factor-covariate dependencies using a scaling factor (\texttt{sf}) - a weight on the factor-covariate term in the objective that can up- or down-weight it relative to the data-factor terms (default $\textrm{sf}=0.1$).

For not paired methods we first calculate all pairs of correlation between covariates and latent factors. We then solve a linear sum assignment problem \cite{hungarian} to align each latent with the covariate that best correlates with it.

\subsubsection{Evaluation: performance metrics}

We evaluate performance using the average factor-covariate correlations, the squared Frobenius distance between the covariance matrix of latent factors $Z$ and the identity matrix, and the average variance explained per factor. 
We also report four disentanglement metrics (Section~\ref{sec:sup_metrics}):
disentanglement ($D$), measuring whether each latent factor encodes information about only one covariate,
completeness ($C$), whether each covariate is captured by only one latent factor,
informativeness ($I$), whether the factors jointly contain the information needed to predict the covariate~\cite{eastwood2018a},
and separated attribute predictability (SAP), the gap between the top and second-best single-latent predictor of each covariate~\cite{kumar2018variational}. All four take values in $[0,1]$, with higher values indicating better disentanglement. We further report normalised SAP, a variant that measures the scale of dominance, that divides SAP by the top predictor's performance.

\section{Results}

\subsubsection{Numerical validation of the theoretical results}
\label{sec:T_FA}

For simulation scenarios (1)-(4) defined above, the optimal transformations $T_\lambda^*$ and $T_{\textrm{ex}}^*$ can be computed explicitly from the  covariances $\Sigma_Z$ and $\Sigma_{ZY}$ (Fig.~\ref{fig:fig1}). Under $T_1^*$ the dependence between pairs $(Z_p, Y_p)$ is strongest, and weaker under $T_{1/2}^*$, $T_0^*$, and $T_{\textrm{ex}}^*$, as stated in Theorem~\ref{lemma1}.
%The comparison with $T_0^*$ follows directly from the optimisation problems: in \eqref{eq:optim_max} the objective is optimised over a broader feasible set than in \eqref{eq:optim_indep}, so $T_1^*$ achieves at least as much dependence as $T_0^*$. 
For $T_0^*$, the reduced alignment is a consequence of enforcing uncorrelated latent factors, a constraint absent under $T_1^*$, which consequently produces correlated factors. 
%This illustrates a more general pattern: imposing constraints on the latent structure reduces the achievable dependence between latent dimensions and covariates. 
A parallel argument applies to $T_{\textrm{ex}}^*$, which imposes exclusive one-to-one factor-covariate dependence at a further cost in alignment. The result illustrates a general principle: constraining the latent structure reduces the achievable dependence between latent dimensions and covariates. Moreover, in contrast to $T_1^*$, under the constraints imposed by $T_{\textrm{ex}}^*$ or $T_{0}^*$, the factor covariance matrix no longer resembles the dependence structure among the covariates. Similar conclusions follow across all scenarios (Figs.~\ref{fig:sup_sim1_fig1}-\ref{fig:sup_sim4_fig1}, panels A, B, and C).

\subsubsection{Post-hoc application of the transformations
}

When pretrained representations are available, the transformations can be computed from empirical estimates $\hat{\Sigma}_Z$ and $\hat{\Sigma}_{ZY}$ and applied post-hoc to the latent representation (Fig.~\ref{fig:big_models}). The proposed
transformations achieve substantially stronger one-to-one dependence between latent dimensions and covariates than the naive baseline of selecting the most correlated dimension of the original representation, and they empirically reproduce the covariate dependence-strength ordering of Theorem~\ref{lemma1}. Applying a Procrustes rotation \cite{book_procrustes, kabsch} directly to the
original embedding yields correlated latent dimensions, because the initial embedding is itself correlated, $T_0^*$, in contrast, first whitens the representation and thus produces uncorrelated dimensions. The transformation $T_1^*$ trades some of this independence for higher covariate alignment, though for all three model embeddings the resulting dimensions remain less correlated than under the plain Procrustes rotation. The covariance matrices (Fig.~\ref{fig:sup_emb}) confirm both the weak
alignment of the initial embedding and that the constraints on the latent imposed by $T_{\textrm{ex}}$ and $T_0$ are satisfied.
Downstream applications that require decoding from the transformed latent, e.g. new sample generation through the original decoder, would apply the inverse transformation to map back to the decoder's expected input space.

\begin{figure}
    \centering
    \includegraphics[width=\linewidth]{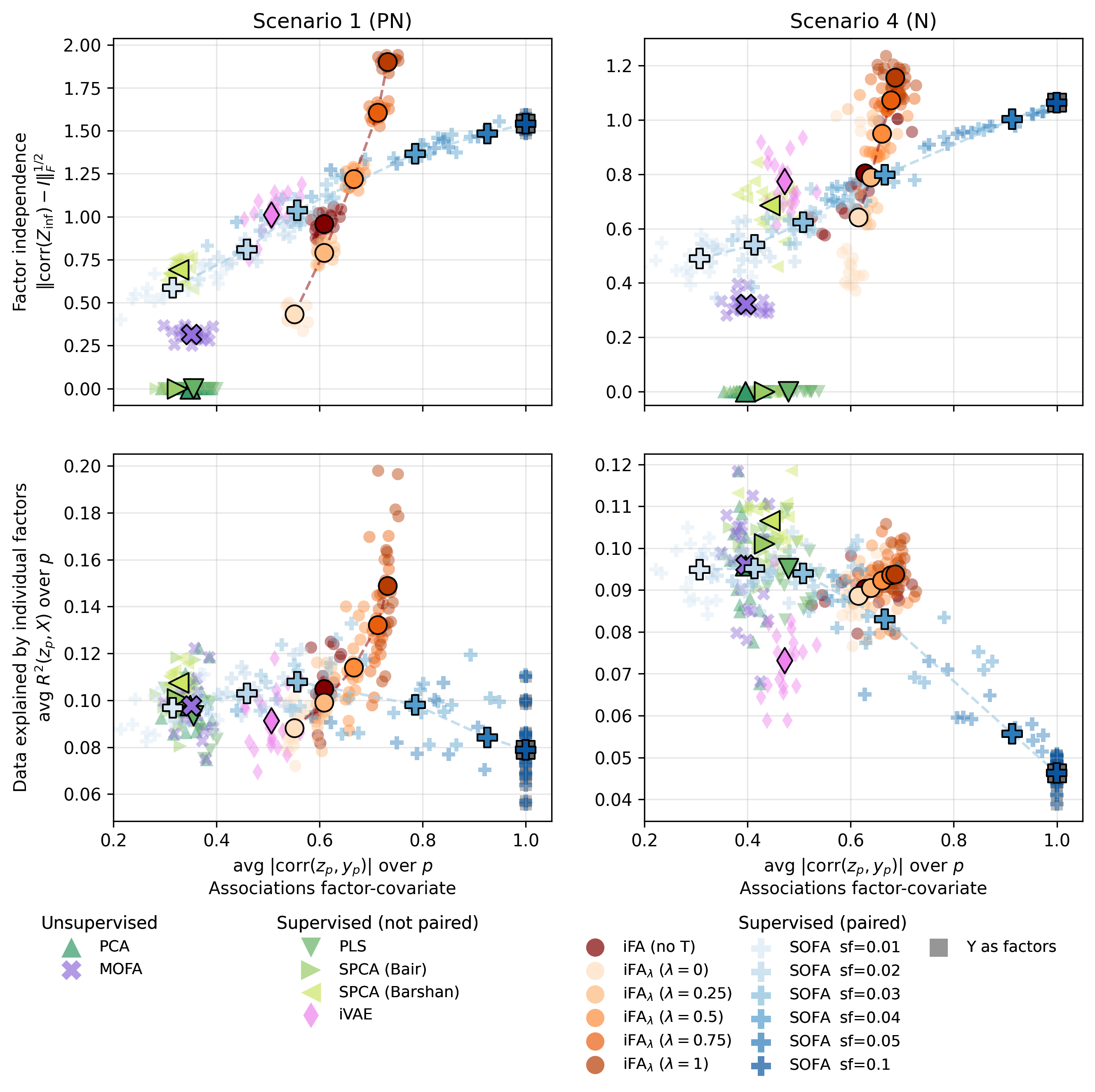}
    \caption{Trade-off between covariate dependence and factor independence for simulation scenarios~1 (PN) and~4 (N). Outlined markers denote means over $20$ repetitions.}
    \label{fig:sim_res}
\end{figure}

\begin{figure*}[h]
    \centering
    \includegraphics[width=\linewidth]{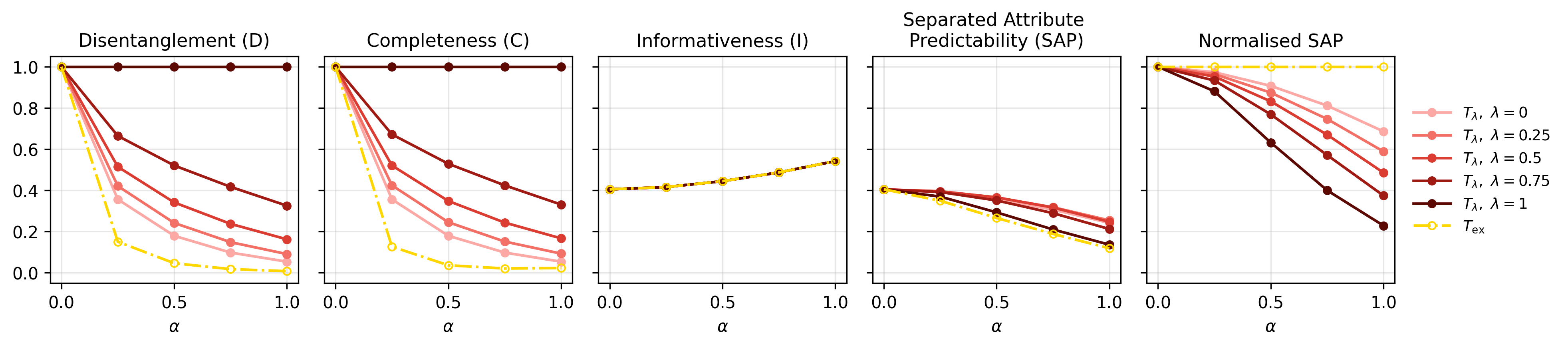}
    \caption{Disentanglement metrics for the latent space transformed using $\,T_\lambda$ family and the exclusive transform $T_{\mathrm{ex}}$ as a function of covariate-correlation strength~$\alpha$ in simulation scenario 1 (PN).}
    \label{fig:dci}
\end{figure*}

\subsubsection{The iFA model controls the trade-off}

Having established the theoretical trade-off between latent-dimension-covariate dependence and latent structure, we now show that the same trade-off emerges when the transformations are folded into the iFA model's inference procedure. 

As a sanity check on the inference procedure, we verify that the pre-transformation variant (\texttt{iFA no $T$}) converges to the true parameters as the number of observations grows: in Scenario~1, both the estimated $\beta$ and the informed latent factors approach their true values (Fig.~\ref{fig:sup_consistency}). Moreover, \texttt{iFA no $T$} recovers the assumed average factor-covariate dependence in every scenario: diagonal correlations concentrate around $\overline{\beta}$, reaching approximately $0.6$ at $b=1$  and $0.4$ at $b=0.6$ (Fig.~\ref{fig:sim_res}; Figs.~\ref{fig:sup_sim1_fig1}-\ref{fig:sup_sim4_fig1}, panel~D).

In Fig. \ref{fig:sim_res}, an ideal method would achieve average covariate correlation of $1$ and Frobenius norm $0$, jointly maximising covariate dependence and factor independence. Theorem~\ref{lemma1} shows this is unattainable in general, as the two objectives trade off. The family of the iFA models with the transformation applied (\texttt{iFA$_\lambda$}) traces this trade-off: as $\lambda$ increases, covariate dependence grows and factor independence weakens (Fig.~\ref{fig:sim_res}; Figs.~\ref{fig:sup_sim1_fig1}--\ref{fig:sup_sim4_fig1}, panel~D). The gain in covariate dependence from moving toward $T_1^*$ is largest when the covariates are strongly inter-correlated, at the cost of factor independence (Figs.~\ref{fig:sup_sim1_fig2}--\ref{fig:sup_sim4_fig2}, panel~A). 
Panel~B of Figs.~\ref{fig:sup_sim1_fig2}-\ref{fig:sup_sim4_fig2} confirms the same pattern across the $\beta$ sweep: correlations of \texttt{iFA$_\lambda$} rise from around $0.2$ to $0.6$ (which are the true generative values), and $T_1^*$ consistently yields stronger covariate dependence and stronger factor inter-dependence than $T_0^*$ across all $\beta$ values. 

Variance explained per factor remains high across all settings and all values of $\lambda$, showing that the transformation reshapes the latent geometry without sacrificing reconstruction or per-factor interpretability of $X$. 

Per-factor interpretability is critical for real-data applications, where individual factors should map onto meaningful directions in the observed features. We confirm this on real data (Fig.~\ref{fig:placeholder}): the variance explained by individual factors remains high across the entire \texttt{iFA$_\lambda$} family, with $T_1^*$ explaining the most and $T_0^*$ the least. The trade-off between latent independence and covariate dependence is preserved throughout.

\begin{figure}[b]
    \centering
    \includegraphics[width=\linewidth]{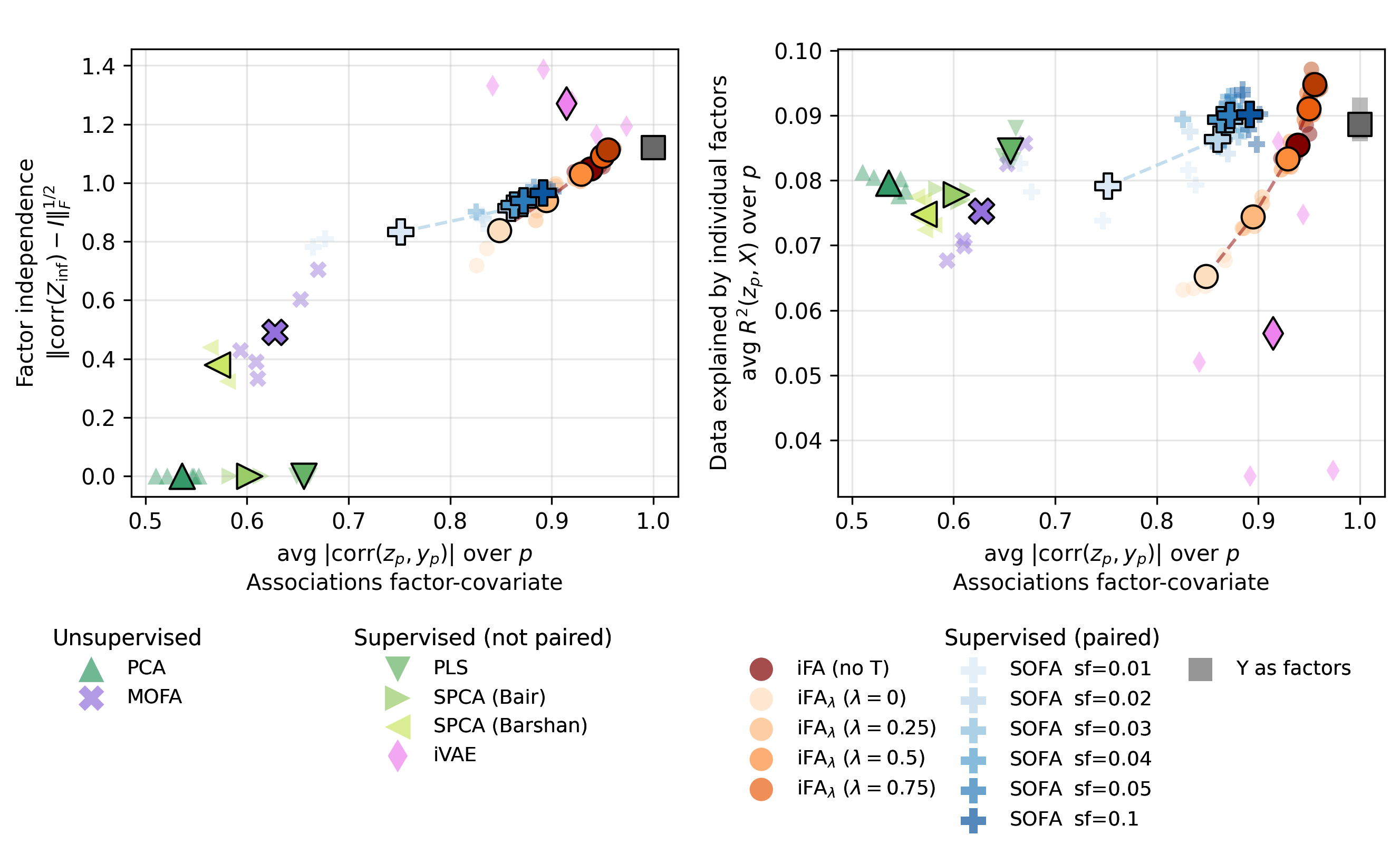}
    \caption{Results on the breast TCGA multi-omics data. Outlined markers denote means over $20$ repetitions.}
    \label{fig:placeholder}
\end{figure}

\subsubsection{Disentanglement-metric view of the trade-off}

For population-level transformations (Fig.~\ref{fig:dci}) and for the \texttt{iFA$_\lambda$} (Figs.~\ref{fig:sup_sim1_fig2}-\ref{fig:sup_sim4_fig2}), disentanglement ($D$) and completeness ($C$) are lower at small $\lambda$ and improve as $\lambda$ grows, with most prominent difference when the covariate correlations are strong. Enforcing latent-dimension independence carries a $D$/$C$ cost: each factor holds less covariate-specific information, and each covariate's information is spread across multiple factors. Informativeness ($I$) tracks signal strength instead - more correlated covariates (larger $b$) and stronger factor-covariate dependencies (larger $\alpha$) both increase $I$ regardless of $\lambda$. 
SAP and its normalised variant, in contrast, favour transformations with small $\lambda$: when factors are constrained to be independent, each covariate is predicted most strongly by a single latent. The two variants disagree on $T_{\mathrm{ex}}^*$: normalised SAP is high as each factor genuinely carries information about only one covariate, by construction, while the unnormalised version is low as the overall dependence magnitude is small (Fig.~\ref{fig:dci}).

\subsubsection{iFA's strengths over competing methods}

Most competing methods fall behind the front traced by the \texttt{iFA$_\lambda$} family (Fig.~\ref{fig:sim_res}, Fig.~\ref{fig:placeholder}, Figs. \ref{fig:sup_sim1_fig1}-\ref{fig:sup_sim4_fig1}). Some unsupervised and supervised baselines achieve strong factor independence at the cost of low covariate dependence: \texttt{MOFA}, \texttt{PCA}, \texttt{PLS}, and \texttt{SPCA (Bair)}, the last three of which produce independent factors by construction. The regression baseline (\texttt{Y as factors}) sits at the opposite extreme, attaining maximal covariate dependence but inheriting the same factor inter-dependence as the covariates themselves. SOFA is the closest competitor in structure: it also parameterises the strength of covariate dependence, via a scaling factor on the objective, however, this parameterisation gives only indirect control, since the scaling factor reweights loss terms rather than the underlying constraints. This lack of steerability causes \texttt{SOFA} to collapse toward \texttt{Y as factors} at high scaling factors. When the covariates are themselves strong factors of variation in $X$, this collapse is benign, but when they are only weakly informative, SOFA loses any advantage over the regression baseline and its factors lose interpretability in terms of $X$. The lack of steerability is also visible on real data where SOFA even at the largest scaling factor achieves weaker covariate dependence than iFA at large $\lambda$ (Fig.~\ref{fig:placeholder}). Across disentanglement metrics, \texttt{SOFA} in general spans a broader range of scores than \texttt{iFA$_\lambda$} (Figs.~\ref{fig:sup_sim1_fig2}-\ref{fig:sup_sim4_fig2}). When covariate correlations are varied, it is \texttt{iFA$_\lambda$} whose scores shift systematically, while \texttt{SOFA}'s remain largely unchanged.

\texttt{iFA$_\lambda$} is slower than the classical baselines but faster than \texttt{SOFA}. Once the base model is fitted, multiple $T_\lambda$ transformations can be applied at negligible cost, making regime exploration cheaper than refitting (Fig.~\ref{fig:sup_time}).

\section{Conclusions}

We introduced a supervised framework for disentangled representation learning that navigates the trade-off between latent-covariate dependence and structural constraints on the latent space. Four disentanglement regimes, independent, intermediate, unconstrained, and exclusive, arise as solutions of optimizations problems, and we proved an ordering of the regimes by latent-covariate dependence (Theorem~\ref{lemma1}). The resulting transformations admit closed-form expressions and apply both post-hoc to pretrained embeddings and inside a new probabilistic factor-analysis model (iFA). Experiments confirm that iFA enable controllability between covariate dependence and latent independence while preserving reconstruction.

%% The acknowledgments section is defined using the "acks" environment
%% (and NOT an unnumbered section). This ensures the proper
%% identification of the section in the article metadata, and the
%% consistent spelling of the heading.
%\begin{acks}
%acknowledgments
%\end{acks}

%%
%% The next two lines define the bibliography style to be used, and
%% the bibliography file.

\bibliographystyle{ACM-Reference-Format}
\bibliography{references}
\clearpage
\newpage
%%
%% If your work has an appendix, this is the place to put it.
\appendix
\renewcommand{\thefigure}{S\arabic{figure}}
\renewcommand{\thetable}{S\arabic{table}}

\section{Proofs and derivations}

\subsection{Reformulation of an orthogonal Procrustes problem \cite{book_procrustes}}
\label{sec:procrustes}
Let $Z \in \mathbb{R}^P$ and $Y \in \mathbb{R}^P$ be both centered so that $\mathbb{E}[Z] = \mathbb{E}[Y] = 0$.
Let $\overline Z \coloneq \Sigma_Z^{-1/2} Z$ denote the whitened vector, which satisfies
$\cov(\overline Z) = \Sigma_Z^{-1/2}\Sigma_Z\Sigma_Z^{-1/2} = I$.
The orthogonal Procrustes problem in population terms is
\[
\min_{Q'Q = I} \ \mathbb{E}\big\| Y - Q\,\Sigma_Z^{-1/2} Z \big\|^2 .
\]
Expanding the squared norm,
\[
\mathbb{E}\big\| Y - Q\,\overline Z \big\|^2
= \mathbb{E}\|Y\|^2 \;-\; 2\,\mathbb{E}\big[Y' Q\,\overline Z\big] \;+\; \mathbb{E}\big[\overline Z{}' Q'Q\,\overline Z\big].
\]
The first term is constant in $Q$, and the last term equals $\mathbb{E}[\overline Z{}'\overline Z] = \tr(I) = P$
since $Q'Q = I$, hence is also constant. Minimizing the objective is therefore
equivalent to maximizing the cross term. Using $\mathbb{E}[ZY'] = \Sigma_{ZY}$ and the
cyclic invariance of the trace,
\[
\mathbb{E}\big[Y' Q\,\Sigma_Z^{-1/2} Z\big]
= \tr\!\big(Q\,\Sigma_Z^{-1/2}\,\mathbb{E}[ZY']\big)
= \tr\!\big(Q\,\Sigma_Z^{-1/2} \Sigma_{ZY}\big).
\]
Consequently,
\[
\min_{Q'Q = I} \ \mathbb{E}\big\| Y - Q\,\Sigma_Z^{-1/2} Z \big\|^2
\quad\Leftrightarrow\quad
\max_{Q'Q = I} \ \tr\!\big(Q\,\Sigma_Z^{-1/2} \Sigma_{ZY}\big),
\]
which is the optimization problem we defined.

\subsection{Closed-form solution of $T_0^*$ \cite{higham_formula}}
\label{sec:sigma_formula}

Let $A = UDV'$ be the singular value decomposition of $A$. Then
\[
A'A = (UDV')'(UDV') = VDU'\,UDV' = VD^2V',
\]
which is an eigendecomposition, so
\[
(A'A)^{-1/2} = VD^{-1}V'.
\]
Multiplying by $A' = VDU'$ gives
\[
(A'A)^{-1/2}A' = VD^{-1}V'\,VDU' = VD^{-1}DU' = VU'.
\]
Hence the maximizer can be written in closed form as
\[
Q^* = VU' = (A'A)^{-1/2}A'.
\]
With $A = \Sigma_Z^{-1/2}\Sigma_{ZY}$, we have $A'A = \Sigma_{ZY}'\Sigma_Z^{-1}\Sigma_{ZY}$
and $A' = \Sigma_{ZY}'\Sigma_Z^{-1/2}$, so
\[
Q^* = \big(\Sigma_{ZY}'\Sigma_Z^{-1}\Sigma_{ZY}\big)^{-1/2}\Sigma_{ZY}'\Sigma_Z^{-1/2},
\]
and therefore
\[
T_0^* = Q^*\Sigma_Z^{-1/2}
= \big(\Sigma_{ZY}'\Sigma_Z^{-1}\Sigma_{ZY}\big)^{-1/2}\Sigma_{ZY}'\Sigma_Z^{-1}.
\]

\subsection{Notes on $T_1^*$: connection to regression.}
\label{sec:notes_t1}

The strong-dependence transformation coincides with multivariate least-squares regression of the
covariates on the latent variables. The population OLS coefficient of $Y_p$ on $Z$ is
$b_p^* = \Sigma_Z^{-1}(\Sigma_{ZY})_{\cdot p}$, so the optimal 
\[t_p^* = \Sigma_Z^{-1}(\Sigma_{ZY})_{\cdot p}/\big((\Sigma_{ZY})_{\cdot p}'\Sigma_Z^{-1}(\Sigma_{ZY})_{\cdot p}\big)^{1/2}
= b_p^*/\|b_p^*\|_{\Sigma_Z},\]
where
$\|x\|_{\Sigma_Z} \coloneq (x'\Sigma_Z x)^{1/2}$,
is exactly the coefficient vector, rescaled so that
$\var(t_p^{*\prime}Z) = 1$. Equivalently $T_1^* = D_1\,\Sigma_{ZY}'\Sigma_Z^{-1}$, with
$\Sigma_{ZY}'\Sigma_Z^{-1}$ the matrix of regression coefficients. The attained alignment is the
multiple correlation coefficient,
\[
J(T_1^*,p) = \sqrt{(\Sigma_{ZY})_{\cdot p}'\Sigma_Z^{-1}(\Sigma_{ZY})_{\cdot p}} = \sqrt{(M^2)_{pp}} = R_p,
\]
the maximal correlation of $Y_p$ with any linear combination of $Z$ (for derivation details of $t_p^*$ for general $Z$ and $J(T_1^*,p)$, see Section~\ref{sec:summary_of_t}).

\subsection{Identifiability}
\label{sec:identifiability}

\begin{proposition}
\label{prop:identifiability}
Fix a regime $\bullet \in \{0, 1, \lambda, \mathrm{ex}\}$ and let $T_\bullet^*(Z)$ denote its optimal
transformation.
For any invertible $A \in \mathbb{R}^{P\times P}$,
replacing the latent by the equivalent representation $Z' = AZ$ yields the same structured
representation:
\[
T_\bullet^*(Z')\,Z' = T_\bullet^*(Z)\,Z.
\]
Hence $\tilde Z_\bullet \coloneq T_\bullet^*(Z)\,Z$ identifies a canonical representation of the latent.
\end{proposition}

\begin{proof}
The objective and constraints of each regime depend on the latent only through the transformed
vector $\tilde Z = TZ$, via $\cov(\tilde Z, Y) = T\Sigma_{ZY}$ and $\cov(\tilde Z) = T\Sigma_Z T'$.
Under $Z' = AZ$ one has $\Sigma_{Z'} = A\Sigma_Z A'$ and $\Sigma_{Z'Y} = A\Sigma_{ZY}$, so the
substitution $T \mapsto T' = TA^{-1}$ is a bijection between transformations of $Z$ and of $Z'$ with
$T'Z' = TZ$, hence $\cov(T'Z', Y) = \cov(TZ, Y)$ and $\cov(T'Z') = \cov(TZ)$. The objective and
feasible set for $Z'$ therefore coincide with those for $Z$ under this bijection, so both problems
have the same optimal representation $\tilde Z$, and $T_\bullet^*(Z')Z' = T_\bullet^*(Z)Z$.
\end{proof}

\begin{corollary}
\label{col:whiteningok}
Taking $A = T_0^*$ in Proposition~\ref{prop:identifiability}, solving the regime in the whitened frame
and composing with the whitening yields the general $Z$ solution:
$T_\bullet^* = T_\bullet^{*}(\overline Z)\,T_0^*$.
\end{corollary}

\subsection{Summary of the defined transformations}
\label{sec:summary_of_t}

We now unify and summarise the transformations defined above, using the following notation. Let
$M \coloneq (\Sigma_{ZY}'\Sigma_Z^{-1}\Sigma_{ZY})^{1/2}$, and let
\[
J_T = J(T) \coloneq \tr\!\big(\cov(TZ, Y)\big)
\]
denote the total latent dimension-covariate alignment, with per-dimension contribution
\[
J_{T,p} = J(T,p) \coloneq \big(\cov(TZ, Y)\big)_{pp}.
\]
Throughout, we assume $\Sigma_Z \succ 0$ and that $\Sigma_{ZY}$ has full rank. Under these
assumptions $M \coloneq (\Sigma_{ZY}'\Sigma_Z^{-1}\Sigma_{ZY})^{1/2}$ is  symmetric,
and positive definite.
Sections~\ref{seq:deriv_t0}-\ref{seq:deriv_tex} derive the general form of each transformation
together with its per-dimension alignment $J(T,p)$. The results are summarised in Table~\ref{tab:summaryt}.

\begin{table*}[h]
\centering
\renewcommand{\arraystretch}{1.9}
\begin{tabular}{@{}cccc@{}}
\hline
Transformation $T$ & (a) $T$ formula (general $Z$) & (b) $d_p$ \small$(D=\diag(d_p))$ & (c) $J_{T,p}$ \\
\hline
$T_0^*$ &
$D M^{-1}\,\Sigma_{ZY}'\Sigma_Z^{-1}$ &
$1$ &
$M_{pp}$ \\
$T_1^*$ &
$D\,\Sigma_{ZY}'\Sigma_Z^{-1}$ &
$\big((M^2)_{pp}\big)^{-1/2}$ &
$\sqrt{(M^2)_{pp}}=\|M_{\cdot p}\|$ \\
$T_\lambda^*$ &
$D\big((1-\lambda)M^{-1}+\lambda I\big)\Sigma_{ZY}'\Sigma_Z^{-1}$ &
$\big(((1-\lambda)I+\lambda M)^2_{pp}\big)^{-1/2}$ &
$\dfrac{(1-\lambda)M_{pp}+\lambda (M^2)_{pp}}
{\sqrt{(1-\lambda)^2+2\lambda(1-\lambda)M_{pp}+\lambda^2(M^2)_{pp}}}$ \\
$T_{\mathrm{ex}}^*$ &
$D\,\Sigma_{ZY}^{-1}$ &
$\big((M^{-2})_{pp}\big)^{-1/2}$ &
$\big((M^{-2})_{pp}\big)^{-1/2}$ \\
\hline
\end{tabular}
\caption{Summary of the transformations, (a) their formulas,  (b) diagonal normalizers, and (c) per-dimension
 correlations with covariates.}
\label{tab:summaryt}
\end{table*}

\subsubsection{Transformation $T_0^*$} 
\label{seq:deriv_t0}

\begin{enumerate}[label=(\alph*)]
    \item Using \eqref{eq:t0}, we get
    \[T_0^* = \big(\Sigma_{ZY}'\Sigma_Z^{-1}\Sigma_{ZY}\big)^{-1/2}\,\Sigma_{ZY}'\Sigma_Z^{-1} = DM^{-1}\,\Sigma_{ZY}'\Sigma_Z^{-1} .\]
\item There is no scaling in $T_0^*$, thus $D=I$.
    \item The cross-covariance induced by $T_0^*$ is
     \begin{align*}
    \textrm{Cov}(T_0^*Z, Y)
      &= T_0^*\Sigma_{ZY} \\
      &= \big(\Sigma_{ZY}'\Sigma_Z^{-1}\Sigma_{ZY}\big)^{-1/2}
         \Sigma_{ZY}'\Sigma_Z^{-1}\Sigma_{ZY} \\
      &= \big(\Sigma_{ZY}'\Sigma_Z^{-1}\Sigma_{ZY}\big)^{1/2}
       = M.
    \end{align*}
    thus
    \begin{align*}
    J(T_0^*, p)= M_{pp}.
    \end{align*}
\end{enumerate}

\subsubsection{Transformation $T_1^*$} 
\begin{enumerate}[label=(\alph*)]
\item From \eqref{eq:t1}, the optimal rows in the whitened form are
\[
t_p^* = \frac{(\Sigma_{\overline{Z}Y})_{\cdot p}}{\|(\Sigma_{\overline{Z}Y})_{\cdot p}\|}.
\]
Collecting them and using $\Sigma_{\overline{Z}Y} = T_0^*\Sigma_{ZY} = M$ gives
 \[T_1^{*}(\overline{Z}) = D_1\, M,\]
 thus
 \[T_1^* =  T_1^{*}(\overline{Z})\, T_0^* = D_1 MM^{-1}\,\Sigma_{ZY}'\Sigma_Z^{-1} = D_1 \,\Sigma_{ZY}'\Sigma_Z^{-1}. \]
 \item The normalizer $D_1$ scales each row to unit norm. Since
$\|(\Sigma_{\overline{Z}Y})_{\cdot p}\| = \|(T_0^*\Sigma_{ZY})_{\cdot p}\| = \|M_{\cdot p}\|$, and
$\|M_{\cdot p}\|^2 = M_{\cdot p}'M_{\cdot p} = (M^2)_{pp}$ (using $M = M'$), we obtain
\[
D_1 = \diag\!\big(\|M_{\cdot p}\|^{-1}\big) = \diag\!\big((M^2)_{pp}^{-1/2}\big).
\]
 \item The cross-covariance induced by $T_1^*$ is    
 \begin{equation*}
    \textrm{Cov}(T_1^*Z, Y) = T_1^*\Sigma_{ZY} =  D_1\, M\, M,
\end{equation*}
so its $p$-th diagonal entry gives the per-dimension alignment
\[J(T_1^*, p)=\sqrt{(M^2)_{pp}}.\]
\end{enumerate}

\subsubsection{Transformation $T_{\lambda}^*$}
\begin{enumerate}[label=(\alph*)]
\item From \eqref{eq:tlambda}, the optimal rows in the whitened form are
\[
t_{p}^* = \frac{\lambda (\Sigma_{\overline{Z}Y})_{\cdot p} + (1-\lambda) e_p}
{\|\lambda (\Sigma_{\overline{Z}Y})_{\cdot p} + (1-\lambda) e_p\|},
\]
thus
\[
T_{\lambda}^{*}(\overline{Z}) = D_{\lambda}\,\big((1-\lambda)I + \lambda M\big),
\]
and composing with the whitening $T_0^* = M^{-1}\Sigma_{ZY}'\Sigma_Z^{-1}$,
\[
T_{\lambda}^* = T_{\lambda}^{(\overline{Z})*}\,T_0^*
= D_{\lambda}\,\big((1-\lambda)M^{-1} + \lambda I\big)\,\Sigma_{ZY}'\Sigma_Z^{-1}.
\]

\item The diagonal entries of the normalizer $D_{\lambda}$ are
$d_{\lambda,p} = \|((1-\lambda)I + \lambda M)_{\cdot p}\|^{-1}$. Expanding the squared norm,
\begin{align*}
&\big\|((1-\lambda)I + \lambda M)_{\cdot p}\big\|^2 \\
&\quad= \big\|(1-\lambda)e_p + \lambda M_{\cdot p}\big\|^2 \\
&\quad= (1-\lambda)^2 + 2\lambda(1-\lambda)\,\langle e_p, M_{\cdot p}\rangle + \lambda^2\|M_{\cdot p}\|^2 \\
&\quad= (1-\lambda)^2 + 2\lambda(1-\lambda)\,M_{pp} + \lambda^2 (M^2)_{pp},
\end{align*}
using $\langle e_p, M_{\cdot p}\rangle = M_{pp}$ and $\|M_{\cdot p}\|^2 = (M^2)_{pp}$. Hence
\[
d_{\lambda,p} = \Big((1-\lambda)^2 + 2\lambda(1-\lambda)M_{pp} + \lambda^2 (M^2)_{pp}\Big)^{-1/2}.
\]

\item The cross-covariance induced by $T_{\lambda}^*$ is
\[
\cov(T_{\lambda}^* Z, Y) = T_{\lambda}^*\Sigma_{ZY}
= D_{\lambda}\,\big((1-\lambda)I + \lambda M\big)\,M,
\]
so its $p$-th diagonal entry gives the per-dimension alignment
\[
J(T_{\lambda}^*, p)
= \frac{(1-\lambda)M_{pp} + \lambda (M^2)_{pp}}
{\sqrt{(1-\lambda)^2 + 2\lambda(1-\lambda)M_{pp} + \lambda^2 (M^2)_{pp}}}.
\]
\end{enumerate}

\subsubsection{Transformation $T_{\mathrm{ex}}^*$}
\label{seq:deriv_tex}
\begin{enumerate}[label=(\alph*)]
\item From \eqref{eq:optim_excl}, the optimal rows are the unit-normalized
rows of $\Sigma_{\overline{Z}Y}^{-1}$,
\[
t_p^* = \frac{(\Sigma_{\overline{Z}Y}^{-1})_{p\cdot}'}{\|(\Sigma_{\overline{Z}Y}^{-1})_{p\cdot}\|},
\]
which gives
\[
T_{\mathrm{ex}}^{(\overline{Z})*} = D_{\mathrm{ex}}\,M^{-1},
\]
and composing with the whitening $T_0^* = M^{-1}\Sigma_{ZY}'\Sigma_Z^{-1}$,
\begin{align*}
    T_{\mathrm{ex}}^* &= T_{\mathrm{ex}}^{*}(\overline{Z})\,T_0^*
= D_{\mathrm{ex}}\,M^{-1}\,M^{-1}\,\Sigma_{ZY}'\Sigma_Z^{-1} \\
&= D_{\mathrm{ex}}\,M^{-2}\,\Sigma_{ZY}'\Sigma_Z^{-1}
= D_{\mathrm{ex}}\,\Sigma_{ZY}^{-1},
\end{align*}
where the last step uses
\[M^{-2}\Sigma_{ZY}'\Sigma_Z^{-1} = (\Sigma_{ZY}'\Sigma_Z^{-1}\Sigma_{ZY})^{-1}\Sigma_{ZY}'\Sigma_Z^{-1}
= \Sigma_{ZY}^{-1}.\]

\item The normalizer $D_{\mathrm{ex}} = \diag(d_{\mathrm{ex},p})$ is
$d_{\mathrm{ex},p} = \|(M^{-1})_{p\cdot}\|^{-1}$. Since $M$ (and $M^{-1}$) is symmetric,
$\|(M^{-1})_{p\cdot}\|^2 = (M^{-1}M^{-1})_{pp} = (M^{-2})_{pp}$, giving
\[
D_{\mathrm{ex}} = \diag\!\big(\|(M^{-1})_{p\cdot}\|^{-1}\big)
= \diag\!\big((M^{-2})_{pp}^{-1/2}\big).
\]

\item The cross-covariance induced by $T_{\mathrm{ex}}^*$ is
\[
\cov(T_{\mathrm{ex}}^* Z, Y) = T_{\mathrm{ex}}^*\Sigma_{ZY}
= D_{\mathrm{ex}}\,\Sigma_{ZY}^{-1}\Sigma_{ZY} = D_{\mathrm{ex}},
\]
which is diagonal, so the cross-covariance constraint holds and the per-dimension alignment is
\[
J(T_{\mathrm{ex}}^*, p) = \big(D_{\mathrm{ex}}\big)_{pp} = \big((M^{-2})_{pp}\big)^{-1/2}.
\]
\end{enumerate}

\subsection{Proof of Theorem \ref{lemma1}}
\label{sec:proof}

(a) The theorem follows by summing the per-dimension inequalities over $p$. The first inequality,
together with its equality condition, is Lemma~\ref{lem:ex_0}. The remaining two, together with the
monotonicity of $\lambda \mapsto J(T_\lambda^*)$, are Lemma~\ref{lem:monotone_lambda}. In lemmas proofs, we
use the notation and per-dimension representations of Section~\ref{sec:summary_of_t}, in particular
that $M = (\Sigma_{ZY}'\Sigma_Z^{-1}\Sigma_{ZY})^{1/2}$. \\
(b) The theorem follows directly from the problem formulations.

\begin{lemma}
\label{lem:ex_0}
For each $p$,
\[
J_{T_{ex}^*,p} \;\le\; J_{T_0^*,p}
\]
with equality if and only if $M_{pp}^2 = \|M_{\cdot p}\|^2$, i.e. the column $M_{\cdot p}$ has no
off-diagonal entries.
\end{lemma}

\begin{proof}
Since $M$ is symmetric positive definite, the two per-dimension values are
$J_{T_0^*,p} = M_{pp}$ and $J_{T_{ex}^*,p} = \big((M^{-2})_{pp}\big)^{-1/2}$, so the claim is
\[
\big((M^{-2})_{pp}\big)^{-1/2} \;\le\; M_{pp}.
\]
We prove it through the intermediate quantity $\big((M^{-1})_{pp}\big)^{-1}$, using the
Cauchy-Schwarz inequality twice.

\emph{Step 1: $\big((M^{-2})_{pp}\big)^{-1/2} \le \big((M^{-1})_{pp}\big)^{-1}$.}
Apply Cauchy-Schwarz to $e_p$ and $M^{-1}e_p$:
\[
(M^{-1})_{pp} = \langle e_p,\, M^{-1}e_p\rangle
\le \|e_p\|\,\|M^{-1}e_p\|
= \big((M^{-2})_{pp}\big)^{1/2},
\]
using $\|e_p\| = 1$ and $\|M^{-1}e_p\|^2 = e_p'M^{-2}e_p = (M^{-2})_{pp}$. Squaring and inverting
gives the claim.

\emph{Step 2: $\big((M^{-1})_{pp}\big)^{-1} \le M_{pp}$.}
Apply Cauchy-Schwarz to $M^{1/2}e_p$ and $M^{-1/2}e_p$:
\begin{multline*}
    1 = \langle e_p, e_p\rangle
= \langle M^{1/2}e_p,\, M^{-1/2}e_p\rangle
\\\le \|M^{1/2}e_p\|\,\|M^{-1/2}e_p\|
= M_{pp}^{1/2}\,(M^{-1})_{pp}^{1/2},
\end{multline*}
so $M_{pp}\,(M^{-1})_{pp} \ge 1$, i.e. $\big((M^{-1})_{pp}\big)^{-1} \le M_{pp}$.

Chaining the two steps,
\[
J_{T_{ex}^*,p} = \big((M^{-2})_{pp}\big)^{-1/2}
\le \big((M^{-1})_{pp}\big)^{-1}
\le M_{pp} = J_{T_0^*,p}.
\]
Equality in Step~1 holds iff $M^{-1}e_p \parallel e_p$, i.e. $e_p$ is an eigenvector of $M$;
the same condition makes Step~2 an equality. This is equivalent to $M_{\cdot p}$ having no
off-diagonal entries, i.e. $M_{pp}^2 = \|M_{\cdot p}\|^2$, and otherwise both inequalities are strict.
\end{proof}

\begin{lemma}
\label{lem:monotone_lambda}
The map $\lambda \mapsto J_{T_\lambda^*,p}$ is non-decreasing on $[0,1]$, thus
\[
J_{T_0^*,p} \;\le\; J_{T_\lambda^*,p} \;\le\; J_{T_1^*,p}, \qquad \lambda \in [0,1].
\]
The inequalities are strict unless $\lambda \in \{0,1\}$ or $M_{pp}^2 = \|M_{\cdot p}\|^2$, i.e. unless
the column $M_{\cdot p}$ has no off-diagonal entries.
\end{lemma}

\begin{proof}
From Section~\ref{sec:summary_of_t}, the per-dimension alignment is
\[
J_{T_\lambda^*,p}
= \frac{\lambda\,\|M_{\cdot p}\|^2 + (1-\lambda)\,M_{pp}}
{\sqrt{\lambda^2\|M_{\cdot p}\|^2 + 2\lambda(1-\lambda)M_{pp} + (1-\lambda)^2}}.
\]
Writing this as $u(\lambda)\,w(\lambda)^{-1/2}$ with numerator $u$ and
$w = \lambda^2\|M_{\cdot p}\|^2 + 2\lambda(1-\lambda)M_{pp} + (1-\lambda)^2$, we have
$\tfrac{d}{d\lambda}J_{T_\lambda^*,p} = w^{-3/2}\big(u'w - \tfrac12 u w'\big)$. Expanding
$u'w - \tfrac12 u w'$ cancels all terms except $(1-\lambda)(\|M_{\cdot p}\|^2 - M_{pp}^2)$, giving
\[
\frac{d}{d\lambda}J_{T_\lambda^*,p}
= \frac{(1-\lambda)\big(\|M_{\cdot p}\|^2 - M_{pp}^2\big)}
{\big(\lambda^2\|M_{\cdot p}\|^2 + 2\lambda(1-\lambda)M_{pp} + (1-\lambda)^2\big)^{3/2}}.
\]
Both factors in the numerator are non-negative on $[0,1]$: $1-\lambda \ge 0$, and
$M_{pp}^2 \le \sum_k M_{kp}^2 = \|M_{\cdot p}\|^2$ since $M_{pp}$ is one entry of the column
$M_{\cdot p}$. Hence $\tfrac{d}{d\lambda}J_{T_\lambda^*,p} \ge 0$, and the endpoint values
$J_{T_0^*,p} = M_{pp}$ and $J_{T_1^*,p} = \|M_{\cdot p}\|$ give the lower and upper bound, respectively. The derivative
vanishes only at $\lambda = 1$ or when $\|M_{\cdot p}\|^2 = M_{pp}^2$, which yields the strictness
condition.
\end{proof}

\section{Additional model details}
Let \(m=1,\ldots,M\) index the modalities in the multi-modal dataset. 
For each modality \(m\), we denote the observed data matrix by 
\(\mathbf{X}^{(m)}\), the corresponding loading matrix by 
\(\mathbf{W}^{(m)}\), and the modality-specific parameters by 
\(\boldsymbol{\alpha}^{(m)}\) and \(\boldsymbol{\tau}^{(m)}\). 
The latent representation \(\tilde{\mathbf{Z}}\) is shared across all modalities. 
The joint distribution factorizes as
\begin{align}
\label{eq:multimodal}
    &p(\{\mathbf{X}^{(m)}\}_{m=1}^{M}, \tilde{\mathbf{Z}}, \{\mathbf{W}^{(m)}\}_{m=1}^{M}, \{\alpha^{(m)}\}_{m=1}^{M}, \{\tau^{(m)}\}_{m=1}^{M} \mid\mathbf{Y}, \beta) = \nonumber \\ \nonumber
    &\quad \prod_{n=1}^{N}\prod_{m=1}^{M}\prod_{d=1}^{D} \mathcal{N}\bigl(x_{n,d}^{(m)}\mid \tilde{z}_{n,\cdot} (w_{d,\cdot}^{(m)})',\, 1/\tau_d^{(m)}\bigr) \\ \nonumber
    &\quad \prod_{n=1}^{N} \prod_{p=1}^{P} \mathcal{N}(z_{n,p}\mid\beta_p^{(0)} + \beta_p y_{n,p}, 1-\beta_p^2)\prod_{n=1}^{N} \prod_{k=P+1}^{K} \mathcal{N}(z_{n,k}\mid 0,1) \\ \nonumber
    &\quad \prod_{m=1}^{M}\prod_{d=1}^{D} \prod_{k=1}^{K} \mathcal{N}(w_{d,k}^{(m)} \mid 0, 1/\alpha_k^{(m)})\prod_{m=1}^{M}\prod_{k=1}^{K}\mathcal{G}(\alpha_k^{(m)} \mid a_0^{(\alpha)}, b_0^{(\alpha)}) \\ 
    &\quad  \prod_{m=1}^{M}\prod_{d=1}^{D} \mathcal{G}(\tau_d^{(m)} \mid a_0^{(\tau)}, b_0^{(\tau)}).
\end{align}

The graphical model for the uni-modal case is shown in Fig.~\ref{fig:PGM}.

\begin{figure}[t]
\centering
\begin{tikzpicture}
  % --- Parameters (rectangles, on the left) ---
  \node[const] (beta) at (-2, 4)   {$\beta_p,\,\beta_p^{(0)}$};
  \node[const] (T)    at (-2, 2.5) {$T_\lambda$};

  % --- Top: y_{n,p} ---
  \node[obs] (y) at (-0.5, 5) {$y_{n,p}$};

  % --- Level 2: z_{n,p} and z_{n,k} ---
  \node[latent] (zp) at (-0.5, 3.5) {$z_{n,p}$};
  \node[latent] (zk) at ( 1.5, 3.5) {$z_{n,k}$};

  % --- Level 3: tilde z, plus alpha and w on the right ---
  \node[det]    (ztp)   at (-0.5, 2)   {$\tilde z_{n,p}$};
  \node[det]    (ztk)   at ( 1.5, 2)   {$z_{n,k}$};
  \node[latent] (alpha) at ( 4.5, 3.5) {$\alpha_k$};
  \node[latent] (w)     at ( 4.5, 2)   {$w_{d,k}$};

  % --- Between tilde z / w and x: tau ---
  \node[latent] (tau) at (2.5, 1.5) {$\tau_d$};

  % --- Bottom: x_{n,d} ---
  \node[obs] (x) at (2.5, 0.5) {$x_{n,d}$};

  % --- Edges ---
  \edge {beta}  {zp};
  \edge {y}     {zp};
  \edge {zp,T}  {ztp};
  \edge {zk}  {ztk};
  \edge {ztp}   {x};
  \edge {ztk}   {x};
  \edge {alpha} {w};
  \edge {w}     {x};
  \edge {tau}   {x};

  % --- Plates (innermost first) ---
  \plate[xshift=-3pt,yshift=3pt, label={[anchor=south west, xshift=2pt, yshift=1pt]south west:{$p=1,\ldots,P$}}] {pP}  {(y)(zp)(ztp)(beta)}              {};
  \plate[xshift=-1pt,yshift=7pt, label={[anchor=north west, xshift=0pt, yshift=0pt]north west:{$k=P+1,\ldots,K$}}] {pKP} {(zk)(ztk)}                 {\textcolor{white}{$k=P+1,\ldots,K$}};
  \plate[xshift=0pt,yshift=-2pt, label={[anchor=south west, xshift=0pt, yshift=0pt]south west:{$k=1,\ldots,K$}}] {pK}  {(alpha)(w)}                {\textcolor{white}{$k=1,\ldots,K$}};
  \plate[xshift=2pt,yshift=0pt] {pD}  {(tau)(w)(x)}               {$d=1,\ldots,D$};
  \plate[scale=1.01, xshift=-3pt,yshift=-2pt, label={[anchor=south west, xshift=2pt, yshift=0pt]south west:{$n=1,\ldots,N$}}] {pN}  {(y)(zp)(ztp)(zk)(ztk)(x)}  {};
\end{tikzpicture}
\caption{Graphical model for iFA (uni-modal case).}
\label{fig:PGM}
\end{figure}
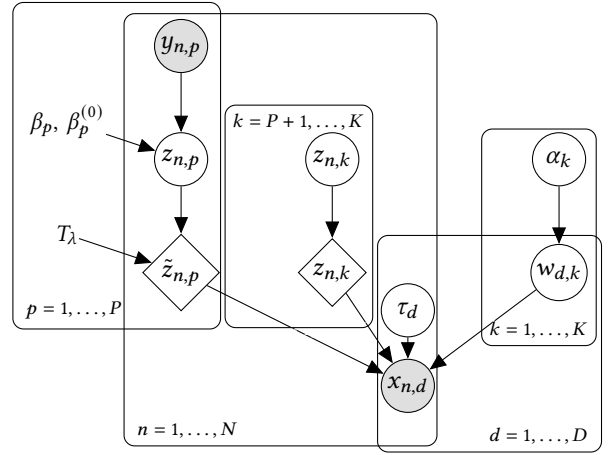

\section{Additional details on the numerical experiments}
\begin{table*}[h]
\centering
\begin{tabular}{c|c|c|c}
Scenario 1 & Scenario 2 & Scenario 3 & Scenario 4 \\
\textbf{Positive \& negative (PN)} & \textbf{Autoregressive (AR)}  & \textbf{Positive (P)} & \textbf{Negative (N)}  \\ \hline 
% --- Row 1 (text / placeholders) ---
- & $(\Sigma_Y)_{ij} = (0.6)^{|i-j|}$ & $(\Sigma_Y)_{ij} =
\begin{cases}
1 & \text{if } i = j \\
0.25 & \text{if } i \ne j
\end{cases}$ & $(\Sigma_Y)_{ij} =
\begin{cases}
1 & \text{if } i = j \\
-0.25 & \text{if } i \ne j
\end{cases}$ \\

% --- Row 2 (FIGURES) ---
\includegraphics[width=0.18\textwidth]{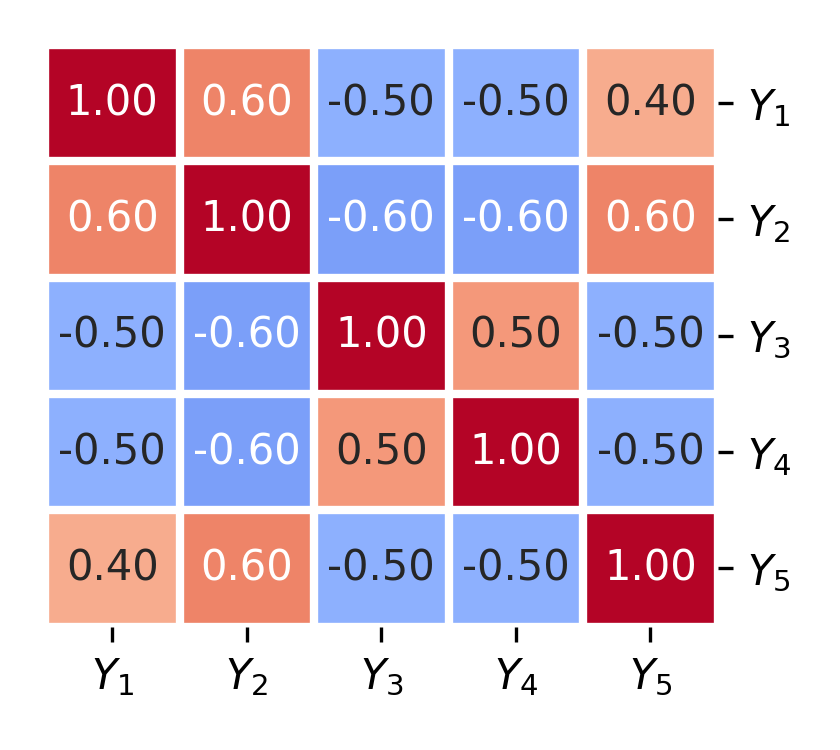} &
\includegraphics[width=0.18\textwidth]{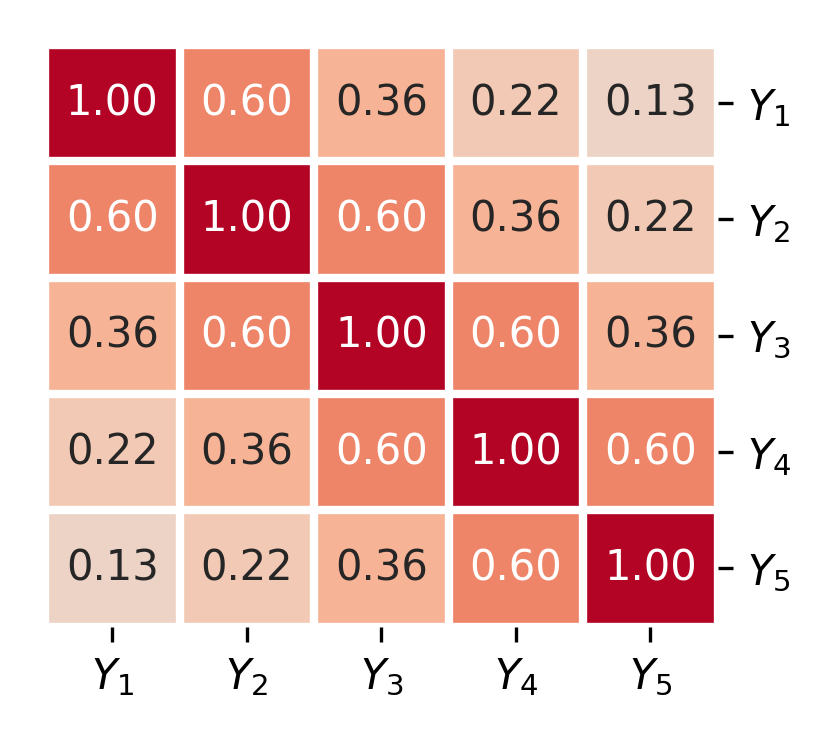} &
\includegraphics[width=0.18\textwidth]{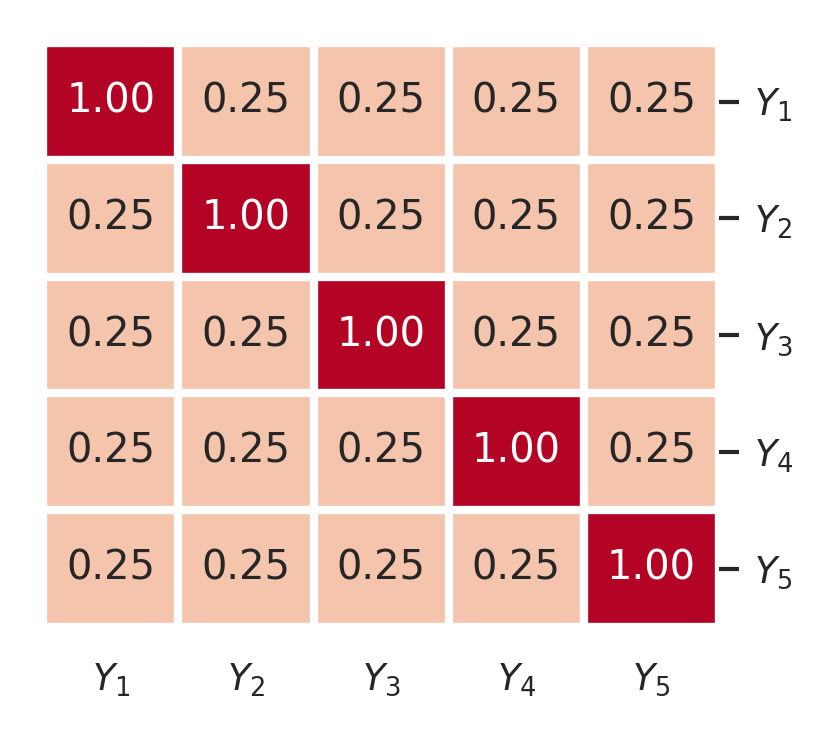} &
\includegraphics[width=0.18\textwidth]{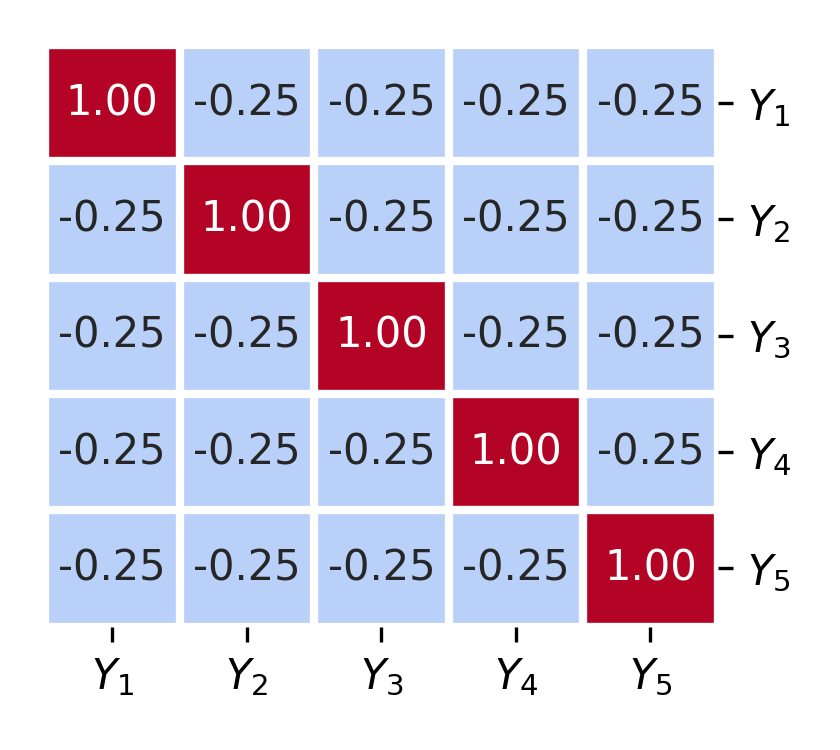} \\

% --- Row 3 (captions or metrics) ---
Unstructured dependencies & Time/space dependences & Scale effects & Dummy-coded categories \\
\end{tabular}
\caption{Simulation scenarios with different covariate dependency structures}
\label{tab:sup_scenarios}
\end{table*}

\subsection{Simulation scenarios}

We generate synthetic data from the generative model of Eq.~\eqref{eq:joint_likelihood}
(Fig.~\ref{fig:PGM}) with $T = I$, varying the covariance structure
of the covariates to reflect qualitatively different but realistic dependence
patterns. The four scenarios with different covariate-covariance structure are summarized in Tab.~\ref{tab:sup_scenarios}, and
the parameters held fixed across all of them in Tab.~\ref{tab:sup_sim_par}. In simulations, we vary two parameters
(Tab.~\ref{tab:sup_vary_par}): $\alpha \in [0, 1]$ controls the strength of the
covariate correlations by interpolating the covariance matrix toward the
identity, and $b \in
[0, 1]$ scales the factor-covariate coefficient vector controlling the strength of the pairwise latent-covariate
dependence.

\begin{table}[h]
\centering
\begin{tabular}{l|l|l}
\textbf{Parameter} & \textbf{Description} & \textbf{Value} \\
\hline
$N$ & number of observations & $500$ \\
D & number of features & 100 \\
$K$ & number of factors & 10 \\
$\beta$ & a vector of coefficients & $(0.9, 0.75, 0.6, 0.45, 0.3)$ \\
$\theta$ & noise to signal ratio & $0.1$ \\
\end{tabular}
\caption{Simulation parameter settings}
\label{tab:sup_sim_par}
\end{table}

\begin{table*}[h]
\centering
\begin{tabularx}{\linewidth}{l X l}
\hline
\textbf{Parameter} & \textbf{Description} & \textbf{Values} \\
\hline
$\alpha \in [0,1]$
&
Controls the strength of covariate dependences.
The covariance matrix is transformed as
${\Sigma_Y^{(\alpha)} = \alpha \Sigma_Y + (1-\alpha)I}$,
with smaller values yielding weaker correlations ($0$ - no correlations, $1$ - strongest correlations).
&
$0,\,0.25,\,0.5,\,0.75,\,1$
\\
$b \in [0,1]$
&
Controls the strength of latent-covariate dependence by scaling
the coefficient vector:
$\beta^{(b)} = b\beta$; the smallest $b$ is selected so that the lowest $\beta_p = 0.1$, then $b$ are spread equilly ($1/3$ - smallest factor-covariate dependence $0.2$ on average, $1$ strongest factor covariate dependence $0.6$ on average)  
&
$1/3,\,1/2,\,2/3,\,3/4,\,1$
% \\
% $\theta > 0$
% &
% Noise-to-signal ratio.
% Larger values correspond to noisier observations.
% &
% $0.01,\,0.05,\,0.1,\,0.15,\,0.2$
\\
\hline
\end{tabularx}
\caption{Simulation parameters varied across experiments.}
\label{tab:sup_vary_par}
\end{table*}

\subsection{Metrics}
\label{sec:sup_metrics}
Let $Z = (z_1, \ldots, z_K)$ denote the latent representation and $Y = (y_1, \ldots, y_P)$ the observed covariates. For each covariate $y_p$, we fit a regressor predicting $y_p$ from $Z$ and obtain a per-latent importance vector; collecting these across $p$ yields the importance matrix $R \in \mathbb{R}^{K \times P}$, where $R_{kp}$ measures the contribution of $z_k$ to predicting $y_p$. On sample data, following~\cite{eastwood2018a}, we set $R_{kp}$ to the absolute Lasso coefficient. For population-level metrics, we exploit the joint covariance of $Z$ and $Y$ and use the closed-form OLS coefficients: $R_{kp} = |(\Sigma_Z^{-1} \Sigma_{ZY})_{kp}|$. We compute these metrics only for the informed (aligned with covariates) latent dimensions.
\begin{itemize}
    \item \textbf{Disentanglement (D)} measures the degree to which each latent dimension $z_k$ encodes information about a single covariate. A latent that contributes to predicting only one $y_p$ is fully disentangled; a latent that contributes to many is entangled. Formally,
    \begin{align*}
        & D = \sum_{k=1}^{K} \rho_k\, D_k, \quad
        D_k = 1 - \frac{H(\tilde{R}_{k,\cdot})}{\log P}, \\
        & \tilde{R}_{kp} = \frac{R_{kp}}{\sum_{q=1}^{P} R_{kq}}, \quad
        \rho_k = \frac{\sum_p R_{kp}}{\sum_{k', p'} R_{k'p'}},
    \end{align*}
    where $H(\cdot)$ is the Shannon entropy of the row-normalised importance vector. A high $D$ means each latent is specialised to one factor and a low $D$ means the latents are entangled across factors. \cite{eastwood2018a}

    \item \textbf{Completeness (C)} \\
    Measures the degree to which each ground-truth factor is captured by a single latent. Formally,
    \begin{align*}
        & C = \sum_{p=1}^{P} \omega_p\, C_p, \quad
        C_p = 1 - \frac{H(\tilde{R}_{\cdot, p})}{\log K}, \\
        & \tilde{R}_{kp} = \frac{R_{kp}}{\sum_{k'=1}^{K} R_{k'p}}, \quad
        \omega_p = \frac{\sum_k R_{kp}}{\sum_{k', p'} R_{k'p'}}.
    \end{align*}
    A high $C$ means each factor is concentrated in one latent, whereas a low $C$ means a factor's information is spread across many latents. \cite{eastwood2018a}

    \item \textbf{Informativeness (I)} measures whether the latent representation actually contains the information needed to recover the ground-truth factors - independent of how that information is organised. Computed as the mean out-of-sample predictive performance of the regressors:
    \begin{equation*}
        I = \frac{1}{P} \sum_{p=1}^{P} R^2(\hat{y}_p, y_p),
    \end{equation*}
    where $\hat{y}_p$ is the prediction of $y_p$ from $Z$. A high $I$ indicates the latents jointly preserve covariate information and a low $I$ indicates information loss. \cite{eastwood2018a}

    \item \textbf{Separated Attribute Predictability (SAP)} \\
    For each pair $(z_k, y_p)$, fit a univariate regression $\hat{y}_p = a_{kp} z_k + b_{kp}$ and compute the predictive score $S_{kp} = R^2(\hat{y}_p, y_p)$. For each ground-truth factor, compute the gap between the highest and second-highest single-feature score, and average over factors:
    \begin{equation*}
        \mathrm{SAP} = \frac{1}{P} \sum_{p=1}^{P} \left( S_{k^{(1)}_p, p} - S_{k^{(2)}_p, p} \right),
    \end{equation*}
    where $k^{(1)}_p$ and $k^{(2)}_p$ denote the latents with the highest and second-highest scores for factor $p$. A high SAP means a single latent stands out as the dominant predictor of each ground-truth factor, and a low SAP means several latents share predictive power for the same factor. Unlike $D$ and $C$, which use the joint-regression importance matrix, SAP measures single-feature predictability and is consequently more sensitive to correlations among the ground-truth factors. \cite{kumar2018variational}
     \item \textbf{Normalised Separated Attribute Predictability (nSAP)} \\
    A scale-invariant variant of SAP, in which the gap is divided by the top predictor's score:
    \begin{equation*}
        \mathrm{nSAP} = \frac{1}{P} \sum_{p=1}^{P}  \left(S_{k^{(1)}_p, p} - S_{k^{(2)}_p, p}\right)/S_{k^{(1)}_p, p}.
    \end{equation*}
    Unlike SAP, which reports the absolute gap and therefore mixes together the strength of the dominant predictor with its lead over the runner-up, nSAP isolates the \emph{relative} dominance of the best predictor. This is particularly useful when overall predictive scores are low: a small SAP gap may reflect either a weak dominant predictor or genuinely competing predictors, whereas nSAP disentangles these two cases.
\end{itemize}

\subsection{Methods details}
Tab.~\ref{tab:sup_methods} lists the methods we compare against, together with
the packages and hyperparameters used. To keep the comparison fair, every method
is given the same number of factors and the same optimisation budget: $K=10$ and
$2000$ iterations (or epochs/SVI steps) for the
simulations, and $K=20$ and $5000$ iterations for the real data. Methods whose
factors are not paired with the covariates by construction (PCA, MOFA, PLS, both
supervised PCA variants, and iVAE) have their informed factors selected by
Hungarian matching on $|\mathrm{corr}(Z,Y)|$, with signs aligned. SOFA and the
regression baseline are paired by design and need no matching. 

Our method (iFA) is fitted with the same $K$ and iteration budget: variational
inference for $2000$ ($5000$ for the real data) iterations, preceded by $250$
($1000$) pretraining iterations of the unsupervised model, and stopped early when
the relative ELBO change falls below $5\times10^{-7}$. The fine-tuning step is run
at $\lambda \in \{0, 0.25, 0.5, 0.75, 1\}$ for the simulations, and $\lambda \in \{0, 0.25, 0.5, 0.75, 0.9\}$ for real data.

\subsection{Image representations}
\label{sec:datasets_images}

 Represenations are extracted with three pretrained
models  obtained through the
HuggingFace \texttt{transformers}:

\begin{itemize}
    \item \textbf{ViT}~\cite{dosovitskiy2021an}, \\
    \url{https://huggingface.co/google/vit-base-patch16-224},
    \item \textbf{CLIP}~\cite{Radford2021LearningTV}, \\ \url{https://huggingface.co/openai/clip-vit-large-patch14},
    \item \textbf{DINOv2}~\cite{oquab2024dinov}, \\ \url{https://huggingface.co/facebook/dinov2-base}.
\end{itemize}
For ViT and DINOv2 we take the model's pooled output, for CLIP we take the
projected image features, each of dimension $768$.

\section{Supplementary figures}

This section collects the complete set of results across all simulation
scenarios, the pretrained representations, and the runtime comparison. The main
text reports Scenarios~1 (PN) and~4 (N). For completeness, some panels shown in
the main text are repeated here:
\begin{itemize}
    \item Fig.~\ref{fig:sup_sim1_fig1}, panels~A and~B also shown in
    Fig.~\ref{fig:fig1};
    \item Fig.~\ref{fig:sup_sim1_fig1} and Fig.~\ref{fig:sup_sim4_fig1},
    right panel of~D also shown in Fig.~\ref{fig:sim_res}.
\end{itemize}

Figure~\ref{fig:sup_emb} shows the latent-covariate and latent-latent
covariance matrices for CLIP, ViT, and DINOv2 under the two baselines and the
proposed transformations.

Figures~\ref{fig:sup_sim1_fig1}-\ref{fig:sup_sim4_fig1} report, for each of the four
scenarios, the data-generating covariance matrices (panel~A), the covariance
matrices after each transformation (panel~B), the trade-off traced by the family
of transformations (panel~C), and the empirical comparison of all methods under
two parameter settings (panel~D).

Figures~\ref{fig:sup_sim1_fig2}-\ref{fig:sup_sim4_fig2} report, for each
scenario, the three proposed metrics (average factor-covariate correlation,
factor independence, and average variance explained per factor) and the five
disentanglement metrics ($D$, $C$, $I$, SAP, nSAP), varying the strength of the
covariate correlations ($\alpha$, panel~A) and of the factor-covariate
associations ($b$, panel~B).

Figure~\ref{fig:sup_consistency} reports parameter recovery as the number of
observations grows: the estimated coefficients, their mean squared error, and
the cosine similarity between recovered and true informed factors.

Figure~\ref{fig:sup_time} compares the runtime of all methods.

\begin{table*}
\centering
\begin{tabularx}{\linewidth}{l X L}
\hline
\textbf{Name} & \textbf{Description} & \textbf{Implementation details} \\
\hline
PCA & \textbf{Principal Component Analysis.} An unsupervised linear method extracting orthogonal directions that maximise variance in $X$. \cite{Pearson01111901, Hotelling1933-gl} & \texttt{PCA(n\_components=10/20)} - \texttt{sklearn 1.8.0} \cite{sklearn} \\
\addlinespace
MOFA & \textbf{Multi-Omicss Factor Analysis} A probabilistic multi-view factor analysis model with a mean-field variational approximation that approximates factor independence. \cite{mofa} &
\texttt{mofapy2 0.7.4} \par
model options: \texttt{factors=10/20, spikeslab\_weights=True, ard\_weights=True,
ard\_factors=True}, \par train options: \texttt{convergence\_mode="medium", iter=2000/5000}\\
\addlinespace
PLS
&
\textbf{Partial Least Squares.}  Extracts latent directions that maximise the covariance between $X$ and $Y$, subject to orthogonality of successive components. \cite{Wold_1975}
& \texttt{PLSRegression(n\_components=10/20)} - \texttt{sklearn 1.8.0} \cite{sklearn}
\\
\addlinespace
SPCA (Bair 2006) &
    \textbf{Screening-based supervised PCA.} A two-step method that first screens features of $X$ by their univariate association with $Y$ (e.g., absolute correlation coefficient above a threshold), then applies standard PCA to the surviving submatrix. The factors are PCA directions of a covariate-relevant subset of features. & Our implementation: features filtered by absolute correlation above the 
 $0.5$ empirical quantile treshold, then \texttt{PCA(n\_components=10/20)} - \texttt{sklearn 1.8.0} \cite{sklearn} \\
 \addlinespace
SPCA (Barshan 2011) &  \textbf{HSIC-based supervised PCA.} Finds directions in $X$ maximising the Hilbert-Schmidt Independence Criterion with $Y$. The objective reduces to a generalised eigenproblem on $X' H L H X$, where $L$ is the $Y$-kernel matrix and $H = I - \frac{1}{n} \mathbf{1}\mathbf{1}'$ is the centring matrix. &  Our implementation: RBF kernel on $Y$ with median bandwidth, linear kernel on $X$, \par
\texttt{scipy 1.17.1} and \texttt{numpy 2.4.6}\\
\addlinespace
iVAE  &
    \textbf{Identifiable VAE.} A variational autoencoder where the prior over $Z$ is conditioned on the auxiliary variables $Y$ via a factorised exponential family, $p(z \mid y) = \prod_p p(z_p \mid y)$. Trained by maximising the conditional ELBO with both an encoder $q(z \mid x, y)$ and a decoder $p(x \mid z)$. \cite{ivae} & Our implementation: \texttt{PyTorch 1.13.1+cu117}, Gaussian variant\par
    encoder, prior and decoder MLPs with $2\times64$ hidden units, \texttt{LeakyReLU}$(0.1)$; Adam, lr$=10^{-3}$, batch$=256$, epochs=$2000$/$5000$. \\
    \addlinespace
SOFA & \textbf{Semi-supervised Factor Analysis} An extension of MOFA that models pairwise factor-covariate dependencies. A scaling factor (\texttt{sf}) reweights the factor-covariate loss term relative to the data-factor terms, providing indirect control over the factor-covariate dependence strength. \cite{sofa} & \texttt{biosofa 0.7.5}, $K=10/20$, llh=\texttt{'gaussian', 'bernoulli'}, n\_steps=$2000$/$5000$, lr=$0.01$ \par \texttt{sf} $\in \{0.01, \dots, 0.05, 0.1\}$. \\
    \addlinespace
Y as factors
&
 \textbf{Regression.} A supervised baseline that takes the covariates themselves as the first $P$ informed factors and fills the remaining $K-P$ slots with the leading principal components of the residual $X - Y \hat{\beta}_y$, where $\hat{\beta}_y$ is the OLS fit of $X$ on $Y$.
&
\texttt{LinearRegression}, \texttt{PCA} - \texttt{sklearn 1.8.0}
\\
\hline
\end{tabularx}
\caption{Comparison methods, with a brief description and implementation details.}
\label{tab:sup_methods}
\end{table*}

\newpage

\begin{figure*}
    \centering
    \includegraphics[width=0.95\linewidth]{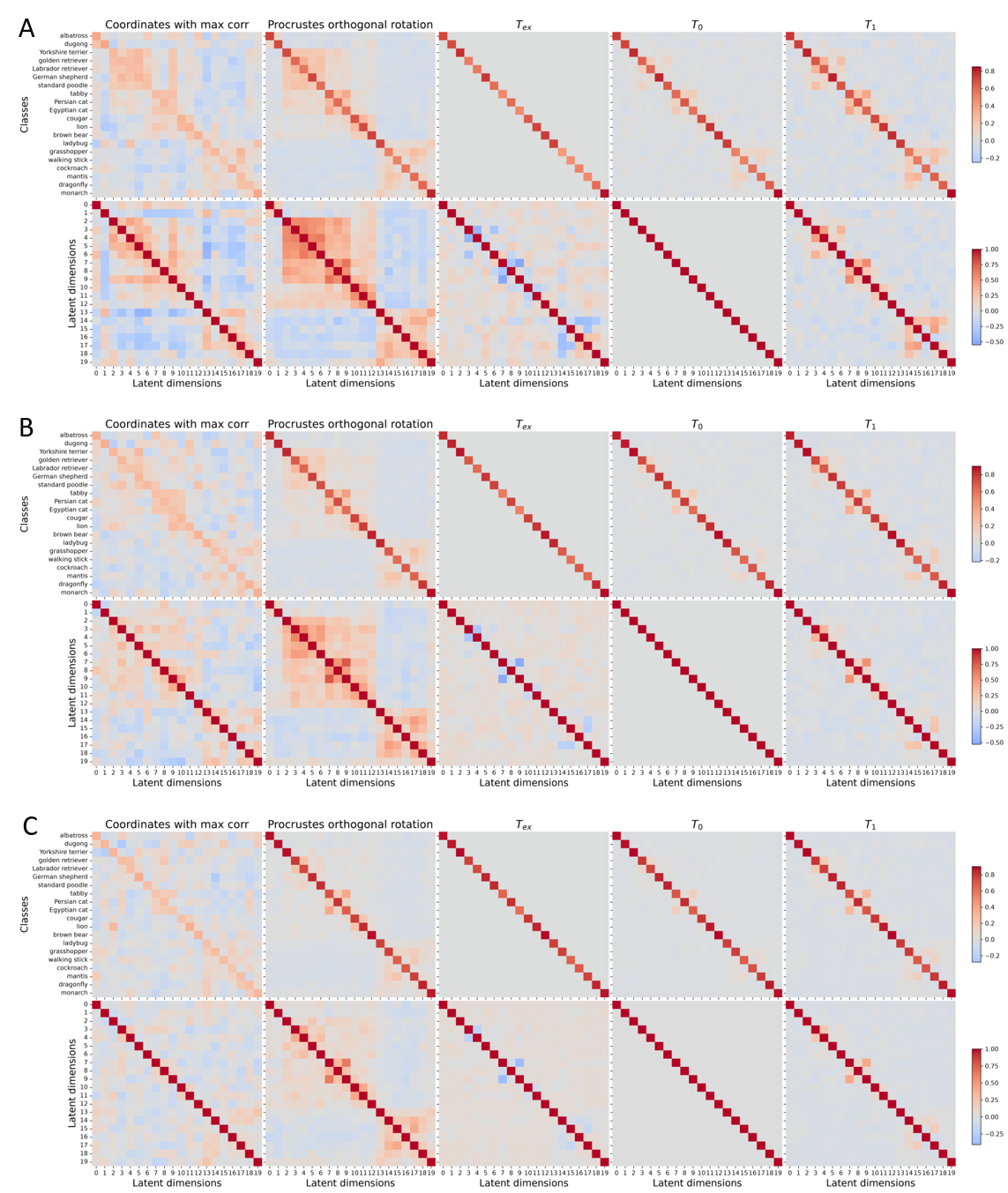}
    % \caption{A. CLIP B. ViT C. DINO}
    \caption{Covariance structure of the pretrained image representations under each transformation, for A.~CLIP, B.~ViT, and C.~DINOv2. Each panel shows latent-covariate (top) and latent-latent (bottom) covariances for the two baselines (dimensions with max correlation, Procrustes rotation) and the proposed $T_{\mathrm{ex}}, T_0, T_1$. }
    \label{fig:sup_emb}
\end{figure*}

\begin{figure*}
    \centering
    \includegraphics[width=0.95\linewidth]{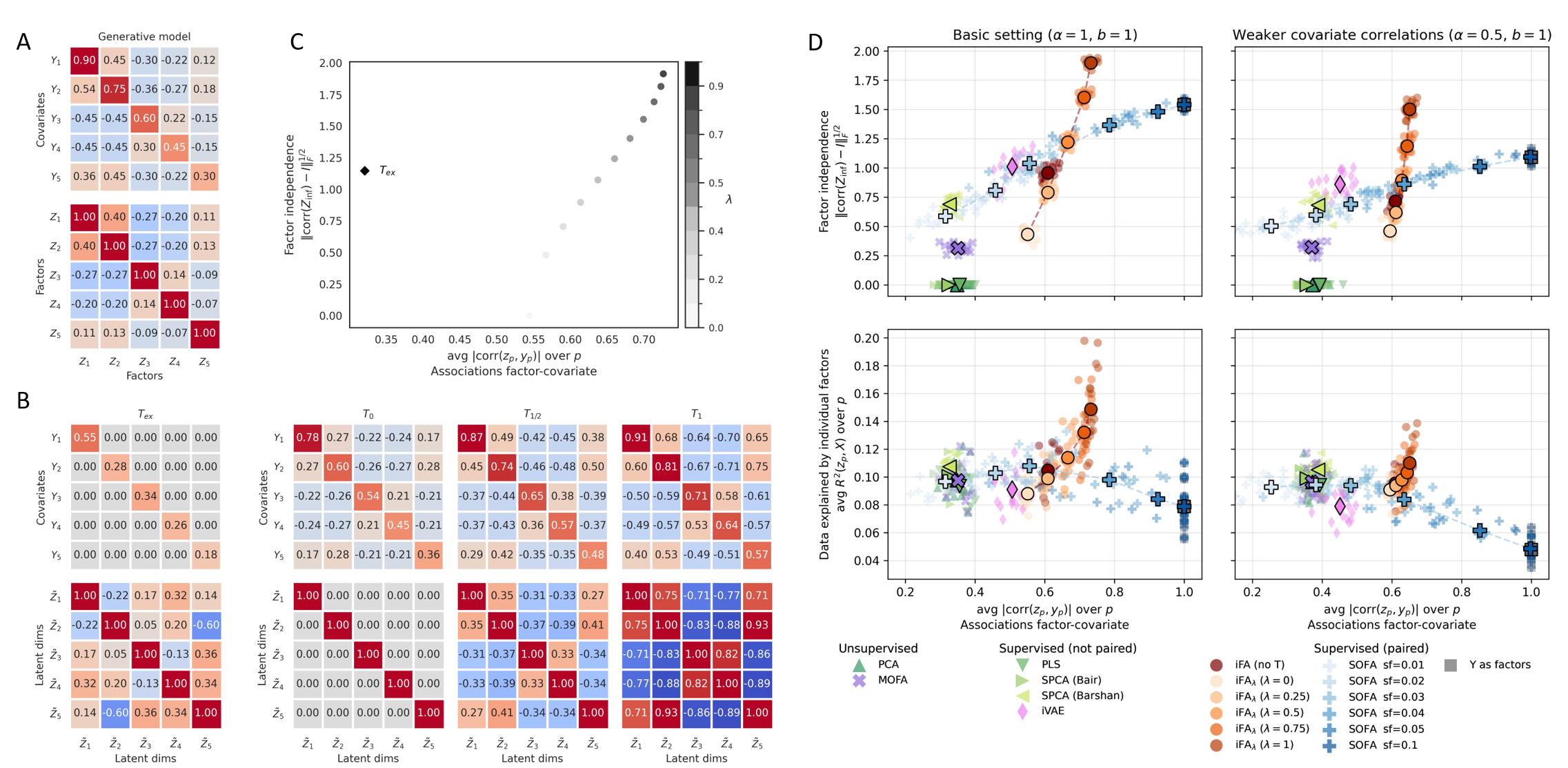}
    \caption{Scenario 1 (PN) results. A. Covariance matrices $\Sigma_{YZ}$ and $\Sigma_{Z}$ in the data-generating setting. B. Covariance matrices $\Sigma_{Y\tilde{Z}}$ and $\Sigma_{\tilde{Z}}$ after transformations $T_{\textrm{ex}}$, $T_0$, $T_{1/2}$, and $T_1$. C. Frontier after transformations $T_\lambda$ and $T_{\textrm{ex}}$. D. Empirical results for methods under two parameter settings: baseline and reduced $\alpha$.}
    \label{fig:sup_sim1_fig1}
\end{figure*}

\begin{figure*}
    \centering
    \includegraphics[width=0.95\linewidth]{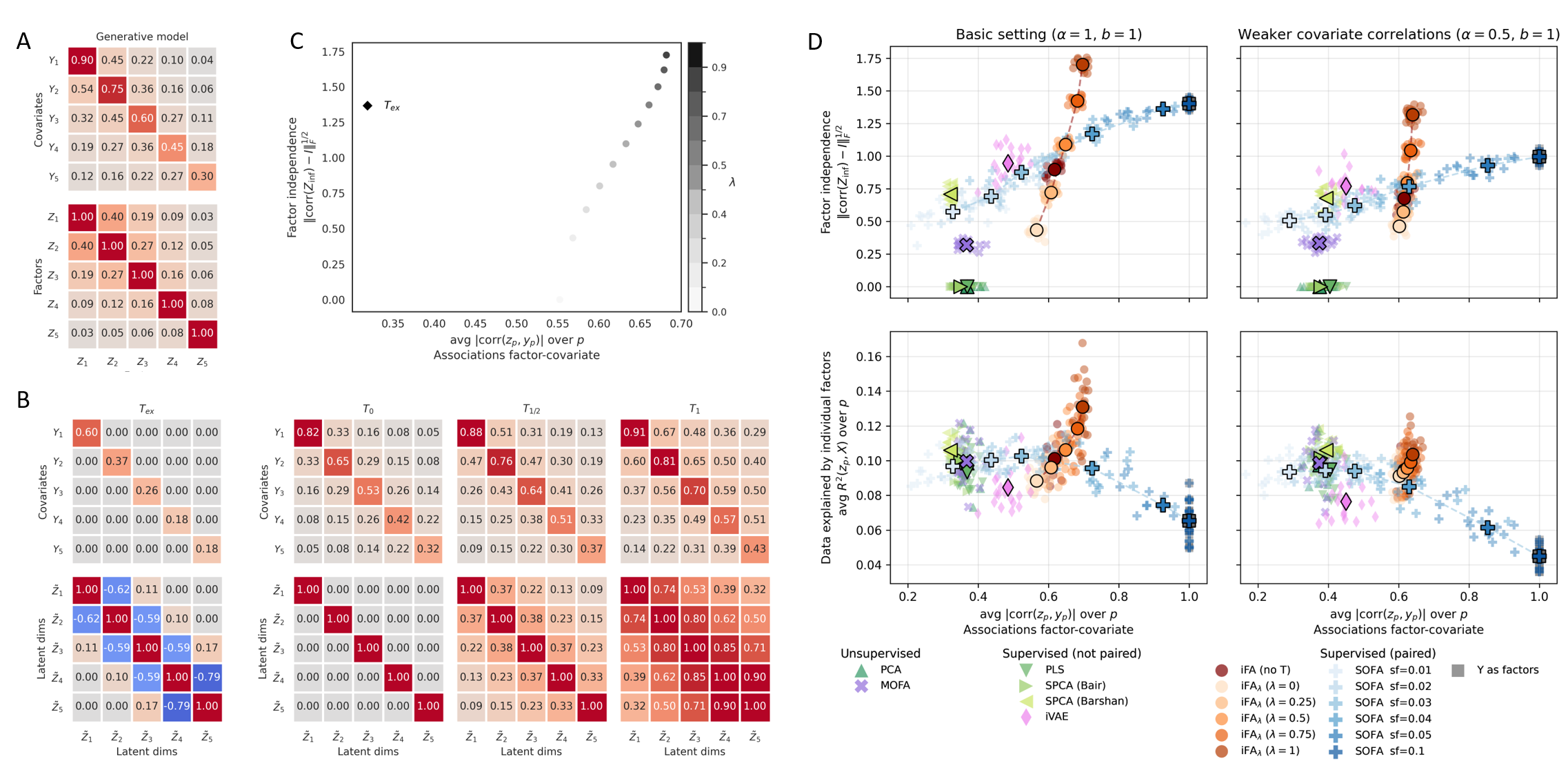}
   \caption{Scenario 2 (AR) results.A. Covariance matrices $\Sigma_{YZ}$ and $\Sigma_{Z}$ in the data-generating setting. B. Covariance matrices $\Sigma_{Y\tilde{Z}}$ and $\Sigma_{\tilde{Z}}$ after transformations $T_{\textrm{ex}}$, $T_0$, $T_{1/2}$, and $T_1$. C. Frontier after transformations $T_\lambda$ and $T_{\textrm{ex}}$. D. Empirical results for methods under two parameter settings: baseline and reduced $\alpha$.}
    \label{fig:sup_sim2_fig1}
\end{figure*}

\begin{figure*}
    \centering
    \includegraphics[width=0.95\linewidth]{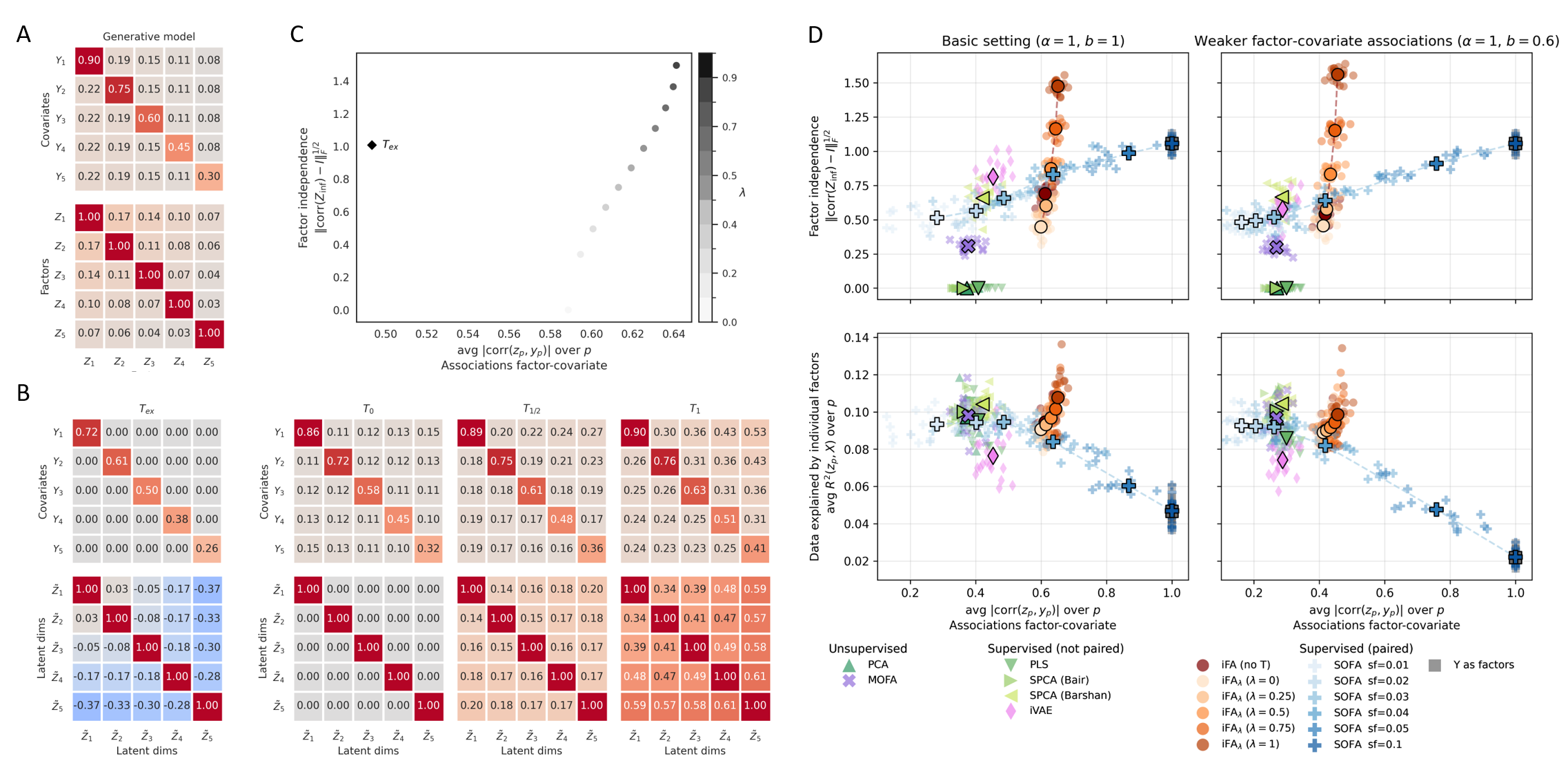}
    \caption{Scenario 3 (P) results. A. Covariance matrices $\Sigma_{YZ}$ and $\Sigma_{Z}$ in the data-generating setting. B. Covariance matrices $\Sigma_{Y\tilde{Z}}$ and $\Sigma_{\tilde{Z}}$ after transformations $T_{\textrm{ex}}$, $T_0$, $T_{1/2}$, and $T_1$. C. Frontier after transformations $T_\lambda$ and $T_{\textrm{ex}}$. D. Empirical results for methods under two parameter settings: baseline and reduced $b$.}
    \label{fig:sup_sim3_fig1}
\end{figure*}

\begin{figure*}
    \centering
    \includegraphics[width=0.9\linewidth]{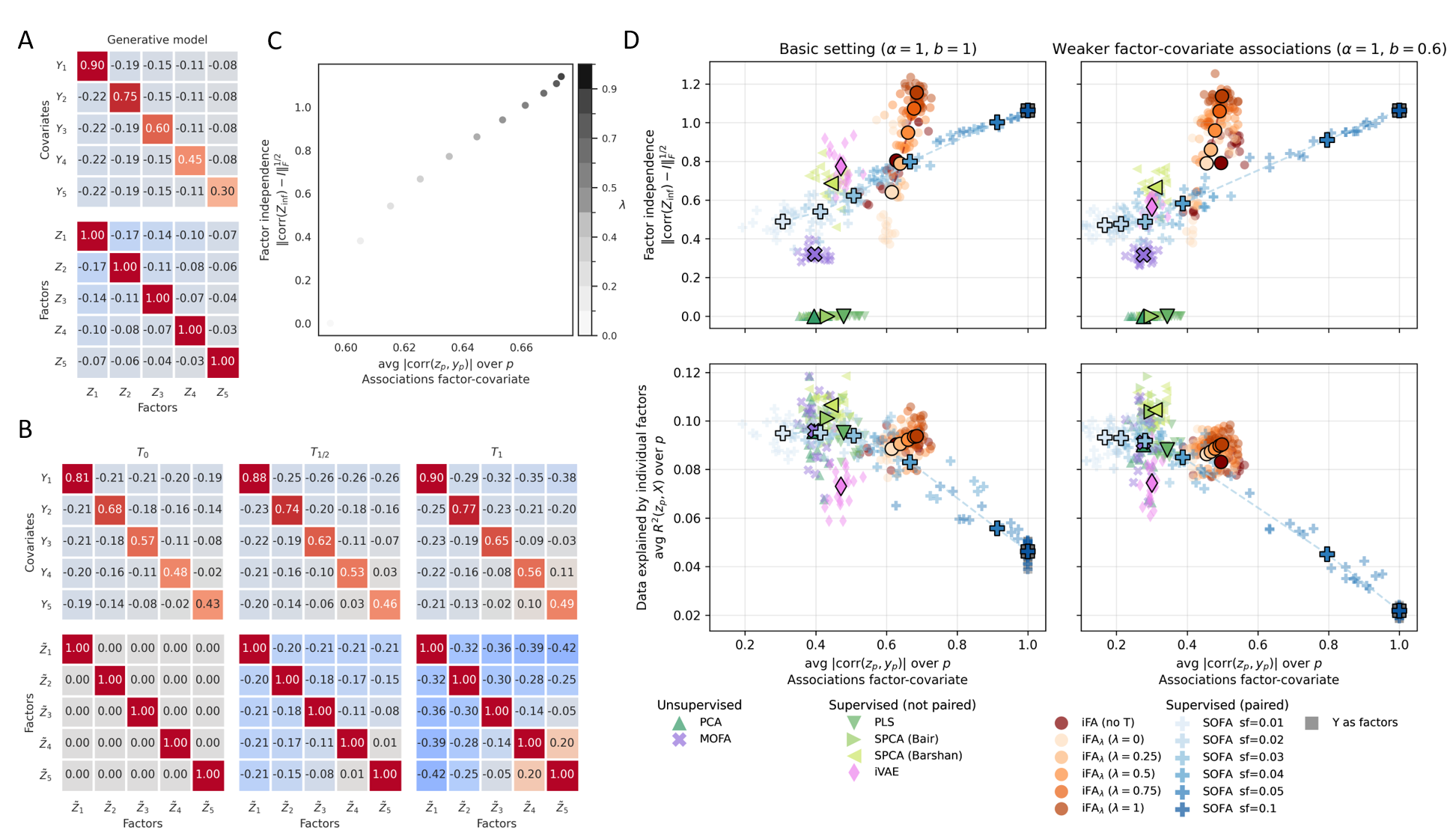}
    \caption{Scenario 4 (N) results. A. Covariance matrices $\Sigma_{YZ}$ and $\Sigma_{Z}$ in the data-generating setting. B. Covariance matrices $\Sigma_{Y\tilde{Z}}$ and $\Sigma_{\tilde{Z}}$ after transformations $T_0$, $T_{1/2}$, and $T_1$ (no $T_{\textrm{ex}}$ is reported, as in Scenario 4 (N) $\Sigma_{ZY}$  is not invertible). C. Frontier of transformations $T_\lambda$. D. Empirical results for methods under two parameter settings: baseline and reduced $b$.}
    \label{fig:sup_sim4_fig1}
\end{figure*}

\begin{figure*}
    \centering
    \includegraphics[width=0.9\linewidth]{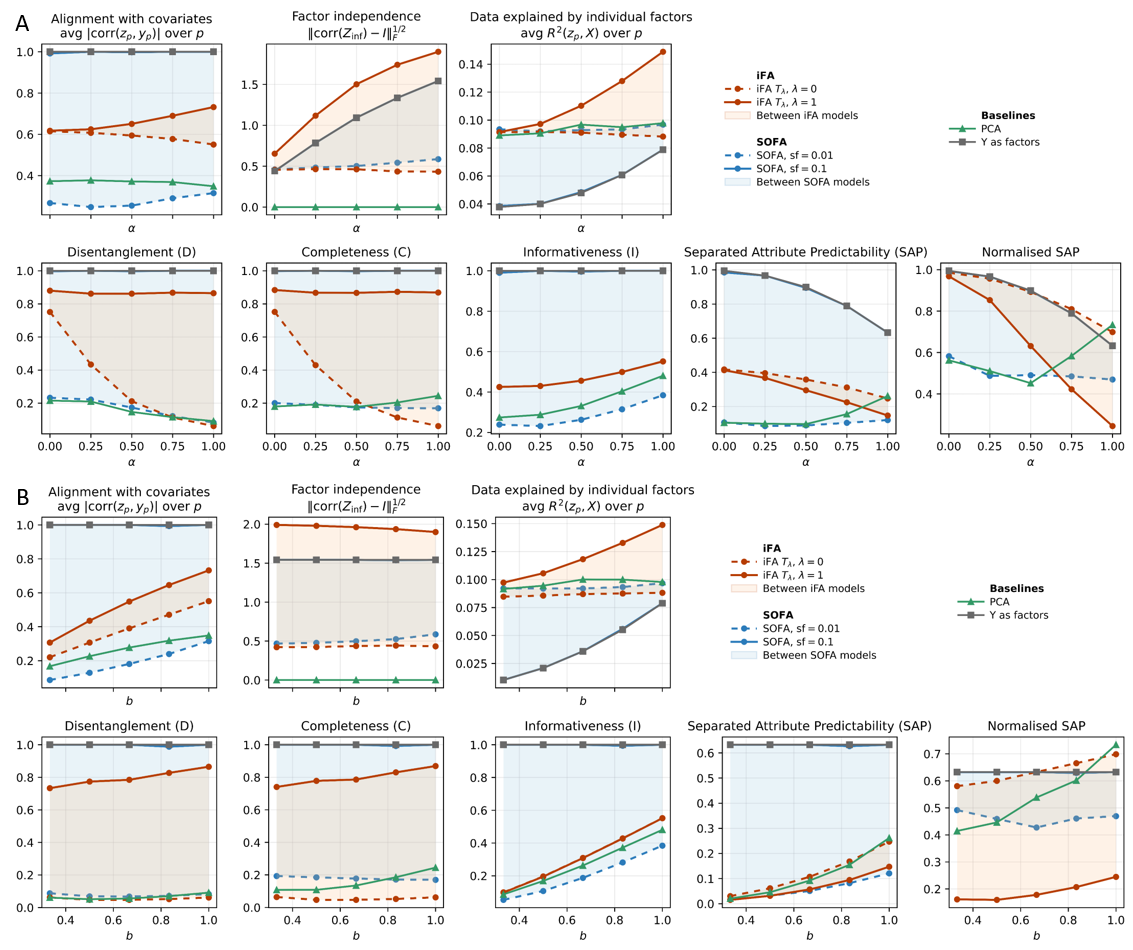}
    \caption{Performance in Scenario 1 (PN), measured by three proposed metrics (average factor–covariate correlations, factor independence measured by the squared Frobenius distance of the correlation matrix of $Z$ from the identity, and average variance explained per factor) and five disentanglement metrics (disentanglement, completeness, informativeness, SAP, and normalised SAP). A. Varying the strength of correlations among covariates. B. Varying the strength of factor–covariate associations.}
    \label{fig:sup_sim1_fig2}
\end{figure*}

\begin{figure*}
    \centering
    \includegraphics[width=0.9\linewidth]{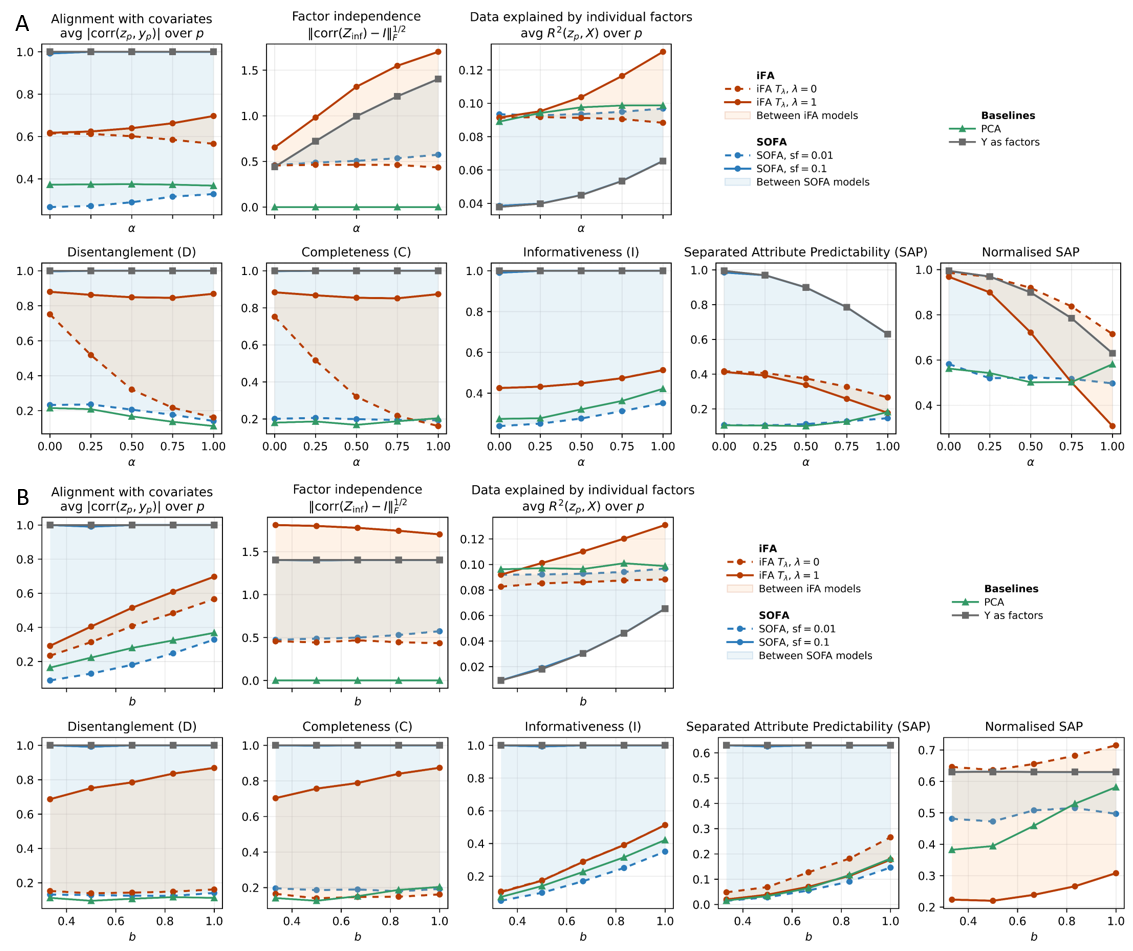}
    \caption{Performance in Scenario 2 (AR), measured by three proposed metrics (average factor–covariate correlations, factor independence measured by the squared Frobenius distance of the correlation matrix of $Z$ from the identity, and average variance explained per factor) and five disentanglement metrics (disentanglement, completeness, informativeness, SAP, and normalised SAP). A. Varying the strength of correlations among covariates. B. Varying the strength of factor–covariate associations}
    \label{fig:sup_sim2_fig2}
\end{figure*}

\begin{figure*}
    \centering
    \includegraphics[width=0.9\linewidth]{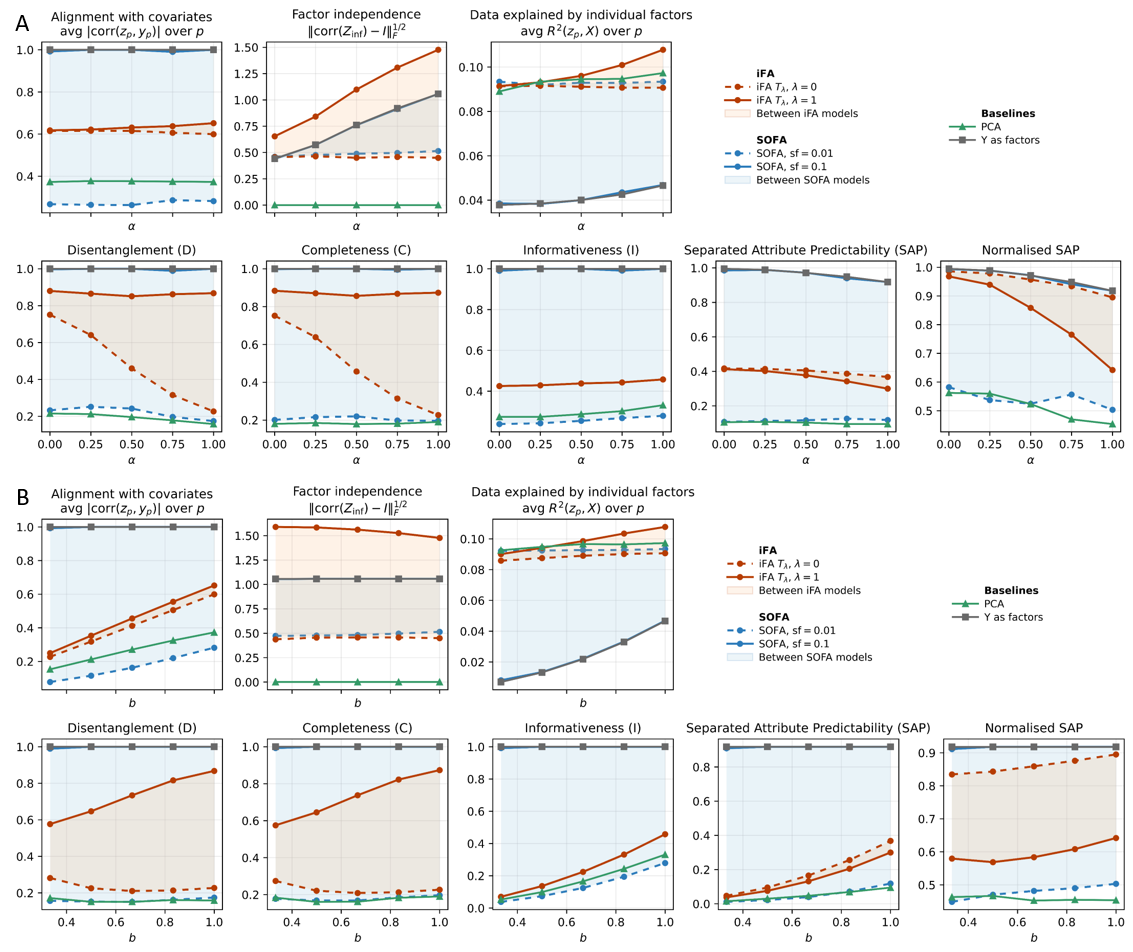}
    \caption{Performance in Scenario 3 (P), measured by three proposed metrics (average factor–covariate correlations, factor independence measured by the squared Frobenius distance of the correlation matrix of $Z$ from the identity, and average variance explained per factor) and five disentanglement metrics (disentanglement, completeness, informativeness, SAP, and normalised SAP). A. Varying the strength of correlations among covariates. B. Varying the strength of factor–covariate associations}
    \label{fig:sup_sim3_fig2}
\end{figure*}

\begin{figure*}
    \centering
    \includegraphics[width=0.9\linewidth]{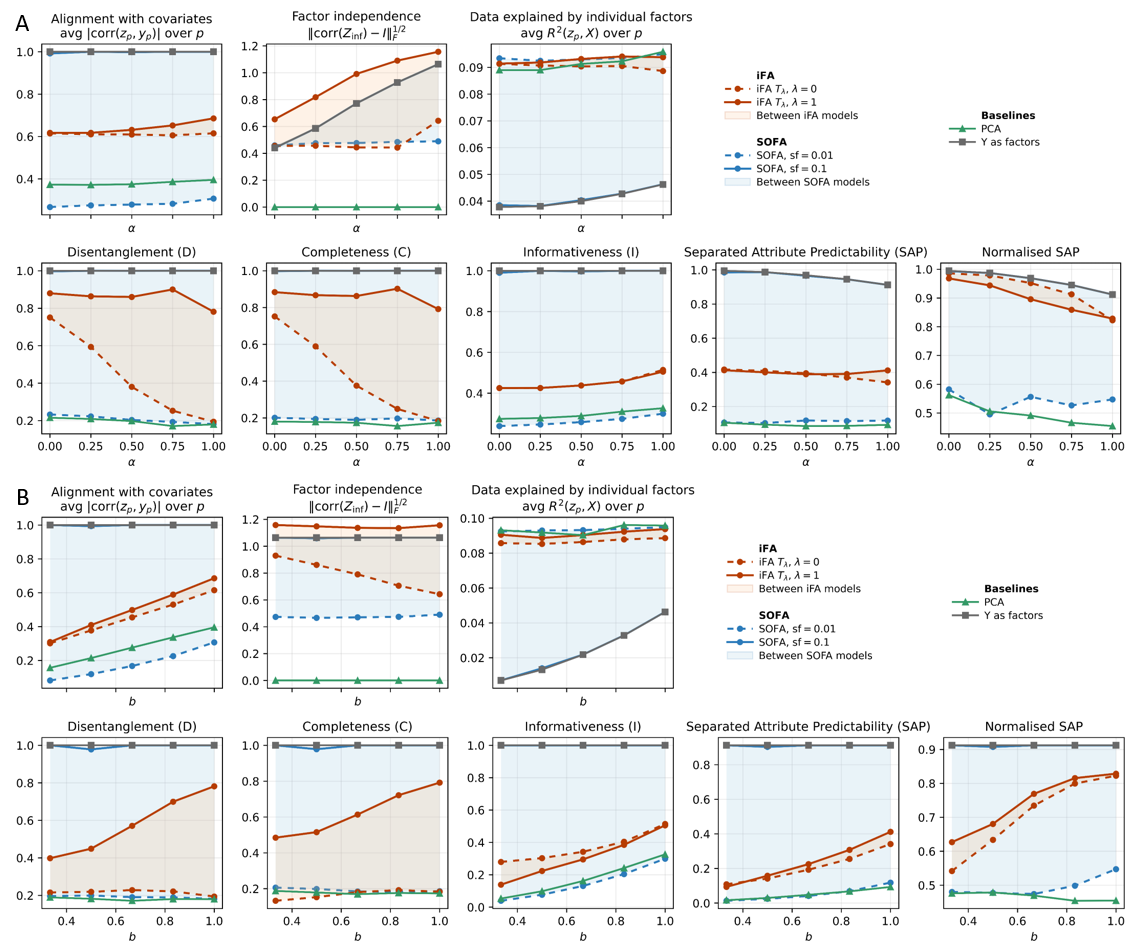}
    \caption{Performance in Scenario 4 (N), measured by three proposed metrics (average factor–covariate correlations, factor independence measured by the squared Frobenius distance of the correlation matrix of $Z$ from the identity, and average variance explained per factor) and five disentanglement metrics (disentanglement, completeness, informativeness, SAP, and normalised SAP). A. Varying the strength of correlations among covariates. B. Varying the strength of factor–covariate associations}
    \label{fig:sup_sim4_fig2}
\end{figure*}

\begin{figure*}
    \centering
    \includegraphics[width=0.95\linewidth]{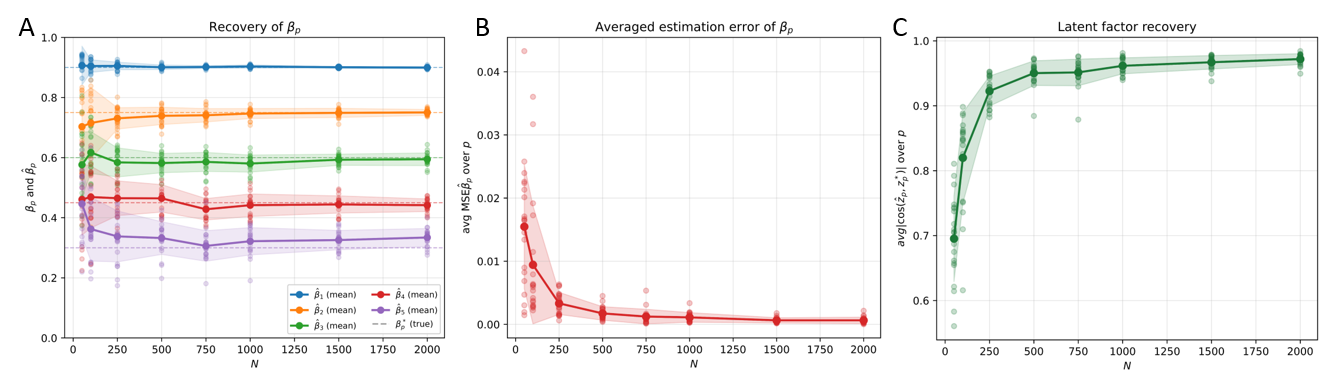}
    \caption{Consistency of iFA (no $T$) as the number of observations~$N$ grows, Scenario~1 (PN). A. Recovery of factor-covariate coefficients. B. Mean squared error of the estimated $\hat\beta_p$ averaged over $p$. C. cosine similarity between recovered and true informed factors. Shaded bands show $\pm1$ s.d. over repetitions.}
    \label{fig:sup_consistency}
\end{figure*}

\begin{figure*}
    \centering
    \includegraphics[width=0.5\linewidth]{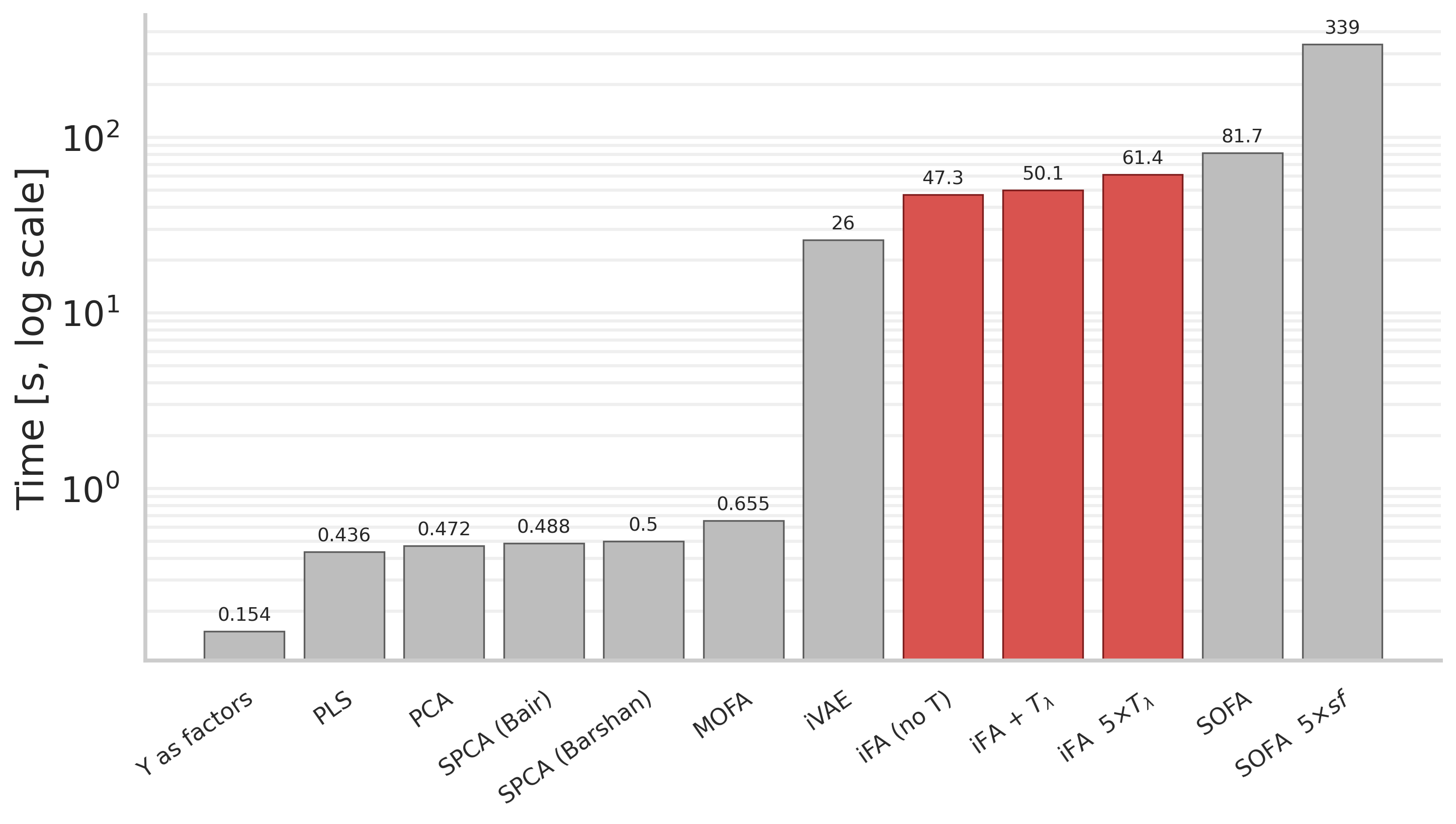}
    \caption{Runtime of all methods on the one simulation setup.}
    \label{fig:sup_time}
\end{figure*}

\end{document}